\documentclass{article}
\usepackage[utf8]{inputenc}

\usepackage{geometry, amsmath, amssymb, mathtools, bbm, amsthm, natbib, xcolor, hyperref, graphicx}
\usepackage[ruled]{algorithm2e} % For algorithms

\usepackage{float}

\usepackage{fullpage}
\usepackage{authblk}
\allowdisplaybreaks

\usepackage{enumerate}

\DeclareMathOperator*{\E}{\mathbb{E}}

\newcommand{\1}{\mathbbm{1}}

\newcommand{\bc}{\mathbf{c}}

\newcommand{\cA}{\mathcal{A}}
\newcommand{\cB}{\mathcal{B}}
\newcommand{\cC}{\mathcal{C}}

\newcommand{\cE}{\mathcal{E}}
\newcommand{\cF}{\mathcal{F}}

\newcommand{\cN}{\mathcal{N}}

\newcommand{\cS}{\mathcal{S}}

\newcommand{\cX}{\mathcal{X}}
\newcommand{\cY}{\mathcal{Y}}

\DeclareMathOperator*{\argmax}{arg\,max}
\DeclareMathOperator*{\argmin}{arg\,min}

\newcommand{\bR}{\mathbb{R}}

\newtheorem{definition}{Definition}

\newtheorem{lemma}{Lemma}
\newtheorem{theorem}{Theorem}
\newtheorem{remark}{Remark}
\newtheorem{corollary}{Corollary}
\newtheorem{proposition}{Proposition}

\newtheorem{assumption}{Assumption}

\newcommand{\BR}{\mathrm{BR}}
\newcommand{\CBR}{\mathrm{CBR}}
\newcommand{\Reg}{\mathrm{Reg}}
\newcommand{\CCV}{\mathrm{CCV}}
\newcommand{\ext}{\mathrm{ext}}
\newcommand{\swap}{\mathrm{swap}}

\usepackage[dvipsnames]{xcolor}

\begin{document}

\title{Online Omniprediction with Long-Term Constraints}

%\author{Anonymous authors}

\author[1]{Yahav Bechavod}
\author[2]{Jiuyao Lu}
\author[1]{Aaron Roth}
\affil[1]{Department of Computer and Information Sciences, University of Pennsylvania}
\affil[2]{Department of Statistics and Data Science, University of Pennsylvania}

\date{\today}

\maketitle

\begin{abstract}
    We introduce and study the problem of online omniprediction with long-term constraints. At each round, a forecaster is tasked with generating predictions for an underlying (adaptively, adversarially chosen) state that are broadcast to a collection of downstream agents, who must each choose an action. Each of the downstream agents has both a utility function mapping actions and state to utilities, and a vector-valued constraint function mapping actions and states to vector-valued costs. The utility and constraint functions can arbitrarily differ across downstream agents. Their goal is to choose actions that guarantee themselves no regret while simultaneously guaranteeing that they do not cumulatively violate the constraints across time. We show how to make a single set of predictions so that each of the downstream agents can guarantee this by acting as a simple function of the predictions, guaranteeing each of them $\tilde{O}(\sqrt{T})$ regret and $O(1)$ cumulative constraint violation. We also show how to extend our guarantees to arbitrary intersecting contextually defined \emph{subsequences}, guaranteeing each agent both regret and constraint violation bounds not just marginally, but simultaneously on each subsequence, against a benchmark set of actions simultaneously tailored to each subsequence.  
    
  %  Equipped with predictions, each agent commits to an action, attempting to simultaneously minimize regret with respect to their respective utility function, as well as refrain from violating a set of incoming adversarial constraints. We present an algorithm that obtains a bound of $\tilde O(|\cA|\sqrt{T})$ on agents' regret, as well as a bound of $|\cA|$ for cumulative constraint violations, where $\cA$ is the set of available actions. Our approach allows, for the first time in the context of learning with long term-constraints, to derive bounds that hold not only with respect to the entire interaction, but rather simultaneously with respect to a sequence of benchmarks that are potentially different for each subsequence. Finally, our approach provides the first algorithm that obtains constraint violation bounds that are independent of the time horizon for adversarial constraints.
\end{abstract}

\section{Introduction} \label{sec:introduction}
The traditional problem of ``online learning from expert advice'' studies a sequential optimization problem that is  decoupled across time. At each iteration $t \in \{1,\ldots,T\}$, an adversary chooses a ``state'' $y_t \in \cY$, and simultaneously, the learner/decision-maker chooses an action $a_t \in \cA$. The learner/decision-maker then experiences utility $u(a_t,y_t) \in [0,1]\footnote{The formalism of choosing a ``state" $y \in \cY$ is fully general; for instance, in the experts problem we can take $\cY = [0,1]^{|\cA|}$ where $y_a$ represents the utility of action $a$.}$. The goal of the learner is to maximize $\sum_{t=1}^T u(a_t,y_t)$, and the standard ``external regret benchmark'' compares the realized utility to the utility that could have been obtained in hindsight by the best -fixed- action: 
$$\Reg_\ext(1:T) = \max_{a \in \cA} \sum_{t=1}^T \left(u(a,y_t) - u(a_t,y_t) \right).$$
In strategic settings, the stronger ``swap regret'' benchmark is often studied because it offers protection against manipulation  (agents with no swap regret guarantees cannot be exploited in the same ways that agents with standard no regret guarantees can be \citep{braverman2018selling,deng2019strategizing,rubinstein2024strategizing,arunachaleswaran2024pareto}):
$$\Reg_\swap(1:T) = \max_{\phi:\cA\rightarrow \cA} \sum_{t=1}^T \left(u(\phi(a_t),y_t) - u(a_t,y_t) \right).$$
Swap regret compares the realized accumulated utility of the learner to the best utility they could have obtained in hindsight by applying a ``swap function'' $\phi$ to their history of actions, consistently mapping each action $a$ to a different action $\phi(a) \in \cA$. External regret is a special case of swap regret in which the ``swap functions'' $\phi$ are restricted to be constant valued. 

These are both ``decoupled'' optimization problems because they are linearly separable across time --- i.e. the ideal solution would always be (if it were possible) to separately optimize every day, and play $a_t \in \argmax_{a \in \cA} u(a,y_t)$.  Regret minimization is a classical problem and there are many algorithms for solving it, obtaining both external and swap regret bounds that scale as $O(\sqrt{T})$, which is optimal \citep{hannan,LittlestoneW94,cesa,arora2012multiplicative,blum2007external}. 

We write ``learner/decision-maker'' above because there are really two problems to solve: \emph{predicting} the value of the next $y_t \in \cY$ and choosing to \emph{act} accordingly by choosing $a_t \in \cA$. Traditional algorithmic approaches combine these two problems and directly learn distributions over actions. But a recent line of work called ``omniprediction'', beginning with \cite{gopalan2021omnipredictors} (initially in a batch/distributional learning setting) proposes to decouple these two problems by producing \emph{predictions} $p_t$ that can be simultaneously consumed by many downstream decision makers who may have different utility functions $u$, while providing all of them with corresponding guarantees. This line of work has been extended to the online adversarial learning setting \citep{noarov2023highdimensional,garg2024oracle,roth2024forecasting,dwork2024fairness,okoroafor2025near,lu2025sample}, in which many downstream decision makers are explicitly promised regret bounds --- sometimes external regret bounds \citep{kleinberg2023u,garg2024oracle,dwork2024fairness,okoroafor2025near} and sometimes stronger swap regret bounds \citep{noarov2023highdimensional,roth2024forecasting,hu2024predict,lu2025sample}. These works explicitly take advantage of the separable nature of the problem: They give algorithms for making predictions $p_t$ (subject to various kinds of ``calibration'' style constraints) such that downstream decision makers can be guaranteed strong regret bounds by individually ``best responding'' to the predictions, each playing an action $a_t = \arg\max_{a \in \cA} u(a,\hat p_t)$ --- i.e. locally optimizing for their own utility function on a day-by-day basis as if the predictions were correct, ignoring the history of outcomes and actions played.

But many online optimization problems are \emph{not} linearly separable in this way. Consider an energy management system where predictions correspond to electricity prices, and actions include: purchasing electricity for immediate use, buying excess power to charge batteries, discharging stored energy, or temporarily shutting down operations. Different downstream decision-makers face distinct constraints: residential users might defer appliance usage to low-cost periods, manufacturers require minimum uptime to meet production quotas, and grant-funded operations must stay within fixed budgets. These constraints couple decisions across time, making prior omniprediction methods inapplicable. Our goal is to give omniprediction style bounds for decision makers like this who are bound to satisfy constraints across time. To crystallize the tension, observe that the goal of omniprediction is to \emph{decouple} prediction from decision making --- but decision makers facing long term constraints must necessarily \emph{couple} their decisions across time, and this coupling behaves differently for differently constrained decision makers.

The literature on ``learning with long-term constraints'' \emph{has} studied different variants of this problem for a single, combined learner/decision-maker who directly learns in action space \citep{MannorTY09,Sun17safety,YuNeely2020,guo2022online,Castiglioni22,qiu2023gradient,sinha2024optimal}. In the problem of learning with long-term constraints, not only does a decision maker have a utility function $u:\cA\times \cY \rightarrow [0,1]$, but they have a set of \emph{constraints} $\{c_j \}_{j \in [J]}$, each of the form $c_j:\cA\times \cY\rightarrow [-1,1]$.\footnote{In the literature on learning with long-term constraints, incoming constraints are usually defined only as a function of the action set. We instead define a fixed set of constraints for each decision-maker, but allow dependence on the label set as well. Hence we can equivalently handle adversarial constraints by encoding adaptivity in the incoming labels.} The goal now is to maximize $\sum_{t=1}^T u(a_t,y_t)$, while satisfying each constraint globally: for each $j \in [J]$, satisfying: $\sum_{t=1}^T c_j(a_t,y_t) \leq 0$. It has been known since \cite{MannorTY09} that in adversarial settings it is not possible to compete with the best fixed action in hindsight that satisfies the constraints marginally. Instead, what has become the standard benchmark is the set of actions that in hindsight satisfy the realized constraint \emph{every round}: $$\cA_{1:T}^\bc = \{a\in \cA : c_j(a,y_t) \leq 0   \ \textrm{for every } t \in [T] \ \textrm{and } j \in [J]\}.$$
The goal in this literature is to minimize external regret with respect to this benchmark:
$$\Reg_\ext(1:T) = \max_{a \in \cA_{1:T}^\bc} \sum_{t=1}^T \left(u(a,y_t) - u(a_t,y_t) \right).$$
while simultaneously minimizing the maximum cumulative constraint violation: 
$$\CCV(1:T) = \max_j \sum_{t=1}^T c_j(a_t,y_t).$$

We introduce the omniprediction variant of this problem, in which a centralized provider of ``predictions as a service'' publishes forecasts $p_t$ of the outcome $y_t$. Many different decision makers, who differ both in their utility functions and constraint functions then consume these predictions to choose actions. Our main result is an algorithm that can simultaneously guarantee each such decision maker regret bounded as $\Reg(1:T) \in \tilde O(\sqrt{T})$ (matching the optimal dependence on $T$, even absent constraints, and defined with respect to the benchmark determined by their own personal constraints) and cumulative constraint violation  $\CCV(1:T) \in O(1)$ --- i.e. independent of $T$. In fact, our techniques yield stronger swap regret bounds at the same rate. In the constrained setting, we define swap regret with respect to the benchmark as:
$$\Reg_\swap(1:T) = \max_{\phi:\cA\rightarrow  \cA_{1:T}^\bc} \sum_{t=1}^T \left(u(\phi(a_t),y_t) - u(a_t,y_t) \right).$$

We also show how these results can be extended to arbitrary collections of \emph{subsequences} $\cS$, where each $S \in \cS$ is a subset of the timesteps $S \subseteq \{1,\ldots,T\}$ that can be determined dynamically by external context. For example, a subsequence $S$ might consist of ``hot days'' that might be correlated with higher electricity prices; the subsequences may intersect arbitrarily. The set of benchmark actions will be different for each subsequence, as some actions may satisfy the constraints at every round in a subsequence $S$ even if they do not satisfy the constraint on every round in $\{1,\ldots,T\}$:
$$\cA_{S}^\bc = \{a\in \cA : c_j(a,y_t) \leq 0   \ \textrm{for every } t \in S \ \textrm{and } j \in [J]\}.$$
We can similarly define regret on a subsequence (to the benchmark $\cA_{S}^\bc$ particular to that subsequence) and cumulative constraint violation on a subsequence:
$$\Reg_\swap(S) = \max_{\phi:\cA\rightarrow  \cA_{S}^\bc} \sum_{t\in S} \left(u(\phi(a_t),y_t) - u(a_t,y_t) \right),$$
$$\CCV(S) = \max_j \sum_{t\in S} c_j(a_t,y_t).$$
Given an arbitrary collection of subsequences $\cS$, we show how to obtain simultaneously for each downstream agent, and simultaneously for each subsequence $S \in \cS$ regret and cumulative constraint violation bounds: $\Reg_\swap(S) \in \tilde O(\sqrt{|S|})$ where $|S|$ is the length of subsequence $S$ and $\CCV(S) \in O(1)$ --- where the constant in the cumulative constraint violation bound now depends linearly on $|\cS|$ (the number of subsequences of interest), but the constant in the regret bound depends only logarithmically on $|\cS|$. Crucially, the benchmark with which regret is measured is tailored to the subsequence --- we are able to compete with a richer set of actions on subsequences that turn out, ex-post to be ``easier'' than others. 
\subsection{Our Results in More Detail}
The outcome space $\cY = [0,1]^d$, and so outcomes $y$ are $d$-dimensional vectors. At every round $t$ our algorithm produces predictions $p_t \in \cY$ that are also $d$-dimensional vectors. Our goal is for downstream decision makers to be able to treat these predictions $p_t$ ``as if'' they were actual outcomes $y_t$. Each decision maker has an action space $\cA$ and is parameterized by a utility function $u:\cA\times \cY \rightarrow [0,1]$ that we assume is linear in $\cY$ --- i.e. such that for every $a$, $u(a,\cdot)$ is a linear function. For a single downstream decision maker, this models the experts learning problem without loss of generality by taking $d = |\cA|$  and letting $u(a,y) = y_a$, but is more general, since we do not restrict either the form of the utility function or $d$ in this way. Each downstream decision maker is also parameterized by a vector valued constraint function $\bc:\cA\times \cY \rightarrow [-1,1]^J$. We assume nothing about its functional form --- e.g. $\bc$ need not be linear or even convex in $\cY$.  
\paragraph{Handling Constraints via Elimination}
In work studying online omniprediction for unconstrained online learning problems \citep{kleinberg2023u,noarov2023highdimensional,roth2024forecasting,garg2024oracle,dwork2024fairness,okoroafor2025near}, downstream decision makers with utility function $u$ ``best respond'' to forecasts, playing action $a_t \in \BR^u(p_t) = \argmax_{a \in \cA} u(a,p_t)$. In our problem of learning with long-term constraints, we similarly study decision makers who play \emph{constrained} best responses. Given a subset of the action space $\cC \subseteq \cA$, the constrained best response to a prediction $p_t$ is $\CBR^u_\cC(p_t)=\argmax_{a \in \cC}u(a,p_t)$ --- i.e. the action in the constrained set $\cC$ that would be utility maximizing were the prediction correct. The question is how decision makers choose the constraint set $\cC$. We study two simple methods, corresponding to two different kinds of benchmarks we can compete with. Most straightforwardly, we study decision makers who start with $\cC = \cA$, and then eliminate actions that have been observed to violate any constraints on any previous round --- i.e. at round $t$, $$\cC = \widehat\cA_{t}^\bc = \{a \in \cA : c_j(a,y_{t'}) \leq 0 \ \textrm{for all }j \in [J], t' < t \}.$$
This is formalized in Algorithm \ref{alg:elimination-realization}. We show in Theorems \ref{thm:CCV-realization} and \ref{thm:swap-regret-realization}  that (when paired with predictions produced by our algorithm), this results in decision makers who cumulatively violate any constraint at most $|\cA|$ times over the course of the interaction, and have diminishing external and swap regret to the standard benchmark $\cA_{1:T}^\bc$. We also consider decision makers who do not eliminate actions the first time they are observed to violate the constraint, but are more conservative, eliminating them only after the action has been observed to have a cumulative constraint violation crossing a well-chosen threshold. This is formalized in Algorithm \ref{alg:elimination-expectation}. We show in Theorems \ref{thm:CCV-expectation} and \ref{thm:swap-regret-expectation} that this lets us compete with a richer benchmark class, $\cA_{1:T}^{\E[\bc]} = \left\{ a: \E_{y_t \sim Y_t}[ c_j(a,y_t)] \le 0 \; \text{ for every } t \in [T] \text{ and } j \in [J] \right\}$,
where here the expectation in the definition of the benchmark is taken over the distribution over outcomes chosen by the adversary at each round --- again both with external and swap regret bounds. In other words, the benchmark allows us to compare to actions that sometimes violate constraints in realization, so long as they do not in expectation over the adversary's randomness. Competing with this richer benchmark class results in a cumulative constraint violation that now grows as $\tilde O(\sqrt{T})$. 

\paragraph{Making Predictions}
Our algorithm for making predictions is parameterized by the collection of utility/constraint functions $(u,\bc)$, one for each downstream decision maker. It makes predictions 
% that are \emph{unbiased} marginally, in the sense that 
% $$\left\| \sum_{t=1}^T (p_t - y_t) \right\|_\infty \leq O(\sqrt{T}),$$ 
% and 
that are unbiased \emph{conditionally on the decisions of each downstream decision maker}, in the sense that for each $(u,\bc)$:
$$\left\| \sum_{t=1}^T \1[\CBR^u_{\widehat\cA_{t}^\bc}(p_t)=a] (p_t - y_t) \right\|_\infty \leq O(\sqrt{T}).$$
This can be accomplished using an appropriate parameterization of the algorithm of \cite{noarov2023highdimensional}, with bias bounds that depend only logarithmically on the number of downstream decision makers $(u,\bc)$ of interest, and running time that depends linearly on the number of downstream decision makers. Similar ``decision calibration'' conditions have been used to give regret bounds in the unconstrained case \citep{noarov2023highdimensional,roth2024forecasting,lu2025sample} --- We show that they continue to give regret bounds in the constrained case (Theorems \ref{thm:swap-regret-realization}, \ref{thm:swap-regret-expectation}) because the elimination method used to eliminate actions by the downstream decision makers maintains the invariant that all of the actions that eventually end up in the benchmark set $\cA^\bc_{1:T}$ are contained in the feasible set $\widehat\cA_t^\bc$ used by the downstream agents at every intermediate step $t \leq T$.

\paragraph{Extension to Subsequences}
Finally, in Section \ref{sec:subsequence} we show how all of our results can be extended to give regret and constraint violation bounds $\Reg(S),\CCV(S)$ for arbitrary collections of context-sensitive \emph{subsequences} $S \in \cS$. 
For the first and simpler benchmark, the idea is simple and requires modifying both the elimination strategy used by the downstream decision makers and the prediction strategy. The downstream decision makers must now maintain sets of feasible actions $\widehat\cA_{t,S}^\bc$ separately, one for each subsequence $S \in \cS$, as before, eliminating actions whenever they have been observed to violate constraints on previous rounds \emph{that are contained in the subsequence $S$}. Given a new round $t$, they then solve the constrained best response problem over the set $\cC = \cup_{S \in \cS : t \in S} \widehat\cA_{t,S}^\bc$ --- i.e. by taking the union of all actions that have not been eliminated on \emph{any} of the subsequences that contain the current round. The predictions that we make must now also satisfy a richer set of conditional bias conditions: informally speaking, they must be unbiased conditionally both on the action chosen by each of the downstream decision makers, and now also membership of the current round in each of the subsequences $S \in \cS$. Once again this can be accomplished via an instantiation of the algorithm of \cite{noarov2023highdimensional} at a cost in bias that is only logarithmic in $|\cS|$ and in running time that is linear in $|\cS|$, the number of subsequences of interest. However, for the second, stronger benchmark, the simple approach described above does not suffice informally because we now don't just count rounds at which constraints were violated, but we sum quantitative constraint violations which can be negative (i.e. on some rounds we can have slack). Consequently, an action can be feasible on a subsequence (and hence never eliminated for that subsequence) while leading to large cumulative constraint violation on intersections of that subsequence with other subsequences. This is problematic for the learner who must give guarantees simultaneously for many intersecting subsequences. 
 The main technical contribution of Section \ref{sec:subsequence} is giving a more sophisticated elimination algorithm that does suffice to bound cumulative constraint violation on all subsequences, despite their intersections.

\subsection{Related Work}
Our work extends both past work on \emph{omniprediction} and \emph{prediction for downstream regret} as well as past work on \emph{online learning with long term constraints}. We discuss both strands of related work below. 
\paragraph{Omniprediction and Prediction for Downstream Regret} \cite{gopalan2021omnipredictors} introduced the problem of \emph{omniprediction} in the batch/distributional setting. In the traditional ``omniprediction'' setting the label space is binary $\cY = \{0,1\}$, the ``action space'' is considered to be the set of real valued predictions in $[0,1]$ and ``decision makers'' are modeled as optimizing a loss function $\ell:[0,1] \times \{0,1\} \rightarrow \mathbb{R}$, such as squared or absolute loss. \cite{gopalan2021omnipredictors} show that multicalibration (c.f. \cite{hebert2018multicalibration} implies omniprediction. Subsequent work \citep{gopalan2022loss,gopalan2023swap} studied both weaker and stronger conditions than multicalibration and the corresponding variants of omniprediction that they imply. \cite{garg2024oracle} introduced the problem of omniprediction in the online adversarial setting and gave oracle efficient algorithms for it; subsequent work derived optimal regret bounds for the problem \citep{okoroafor2025near,dwork2024fairness}. All of the papers discussed so far study binary outcomes $(\cY = \{0,1\})$, but a notable exception is  \cite{gopalan2024omnipredictors} who study vector valued outcomes and decision makers with loss functions that are convex in the outcome. Two papers \citep{globus2023multicalibrated,hu2023omnipredictors} study the problem of omniprediction with constraints. Like most of the rest of the literature, they focus on binary outcomes $\cY = \{0,1\}$ and real valued action spaces $\cA = [0,1]$ in the batch setting. As their focus is on machine learning, they both focus on specific kinds of constraints that arise as ``group fairness'' constraints in binary classification. 

A parallel literature \cite{kleinberg2023u,noarov2023highdimensional,roth2024forecasting,hu2024predict} has focused on making sequential predictions of adversarially chosen outcomes from some space $\cY$ so that downstream agents, who may have arbitrary discrete action spaces, can optimize their own loss. Several papers from this literature \cite{noarov2023highdimensional,roth2024forecasting} have gone beyond the binary label setting and have focused (as we do) on arbitrary loss functions that for each action are \emph{linear} in the state $y \in [0,1]^d$ to be predicted. The goal is to guarantee that all downstream decision makers have diminishing regret in the worst case. Motivated by decision makers interacting in competitive environments,  \cite{noarov2023highdimensional,roth2024forecasting,hu2024predict} give results to guarantee downstream decision makers diminishing \emph{swap} regret at near optimal rates. 
    Because calibration is known to be unobtainable in online adversarial settings at $O(\sqrt{T})$ rates \citep{qiao2021stronger,dagan2024improved}, this literature has from the beginning used techniques (like ``U-calibration'' \citep{kleinberg2023u} and extensions of decision calibration \cite{zhao2021calibrating,noarov2023highdimensional,roth2024forecasting}) that sidestep these lower bounds while  being strong enough to guarantee downstream decision makers no (swap) regret. \cite{lu2025sample} give a unifying view on both literatures. Our model and techniques derive mainly from the ``prediction for downstream regret literature'' (arbitrary action space, high dimensional outcome space, swap regret guarantees in an online adversarial setting) --- one way to view our work is as extending omniprediction with constraints \citep{hu2023omnipredictors,globus2023multicalibrated} to this setting. 

\paragraph{Online Learning with Long-Term Constraints}
The problem of learning with long-term constraints first appeared in \cite{MannorTY09}, who showed that even for simple utility functions and a single adaptive constraint, it is impossible for the learner to get sublinear external regret with respect to the benchmark of actions that satisfy the constraint \emph{marginally}. Instead, follow up work (e.g. \cite{Sun17safety,guo2022online,anderson2022lazy,yi2023distributed,sinha2024optimal,lekeufack2024optimistic})  adopted a stricter external regret benchmark, of actions that satisfy the constraints \emph{on every round}. Generally speaking, these works operate in the setting of \emph{online convex optimization} rather than considering a finite set of actions. Among these, \cite{sinha2024optimal} obtained the fastest known convergence rates without additional assumptions, scaling with $\tilde{O}(\sqrt{T})$ for each of the regret and the cumulative constraint violation. \cite{lekeufack2024optimistic} provided an algorithm for the experts setting with rates scaling each with $\tilde{O}(\sqrt{T\ln(d)})$ where $d$ is the number of experts, and regret is measured against the set of all probability vectors over the experts that (in expectation) satisfy the constraints. This literature also explores benchmarks where constraints are not violated \emph{in expectation} \citep{neelyYu2017,Wei20}, in cases where there assumed to exist an underlying distribution incoming constraint functions. This can be viewed analogously to $\cA_{1:T}^{\E[\bc]}$ in our formulation, for the case where the distribution over states also encodes a distribution over constraints that is fixed across the rounds of interaction. Our formulation, however, is more permissive as it allows the distribution over constraints to vary across rounds.

\section{Model and Preliminaries} \label{sec:preliminaries}
Let $\cX$ denote the feature space and $\cY$ denote the outcome/label space. Throughout, we consider $\cY = [0,1]^d$. 
% (more generally, we can consider any convex, bounded $d$-dimensional space with a constant degradation in our bounds). 
% We let $\cP$ denote the prediction space. 
We consider a set of agents $\cN$ with an arbitrary action space $\cA$. Each agent is equipped with a tuple $(u,c_1,\ldots,c_J)$, which includes a utility function $u: \cA \times \cY \to [0,1]$ and $J$ constraint functions $\{c_j: \cA \times \cY \to [-1,1]\}_{j \in [J]}$\footnote{Sometimes online adversarial learning problems are described by an adversary choosing a different utility and/or constraint function at each step. This is equivalent to having a fixed state-dependent utility function/constraint functions, and having an adversary choose state.}.
We also sometimes write the constraint functions as a single vector valued function $\bc = (c_1,\ldots,c_J)$. 
Since agents are uniquely defined by their corresponding tuples, we treat agents and their tuples interchangeably.
We assume that the utility functions are linear and Lipschitz-continuous in $y$. 
\begin{assumption}
    Fix a utility function $u: \cA \times \cY \to [0,1]$. We assume that for every action $a \in \cA$, $u(a, y)$ is linear in $y$, i.e. $u(a, k_1 y_1 + k_2 y_2) = k_1 u(a, y_1) + k_2 u(a, y_2)$ for all $k_1,k_2 \in \bR$, $y_1,y_2 \in \cY$. Moreover, we assume that $u(a,y)$ is $L$-Lipschitz in $y$ in the $\ell_\infty$ norm: for any $y_1,y_2 \in \cY$, $|u(a, y_1) - u(a, y_2)| \leq L\|y_1 - y_2\|_\infty$. 
\end{assumption}

\begin{remark}
    For simplicity we assume that the utility functions are \emph{linear} in $y$, but we can equally well handle the case in which the utility functions are \emph{affine} in $y$, as we can augment the label space with an additional coordinate that takes constant value 1. This preserves the convexity of the label space and allows for arbitrary constant offsets in the utility of each action.

    Assuming linear/affine utility functions is only more general than the standard assumption that decision makers are \emph{risk neutral} in the sense that in the face of randomness, decision makers act to maximize their expected utility. If $\cY$ represents the set of probability distributions over outcomes, any risk neutral decision maker has a linear utility function by linearity of expectation. 
\end{remark}

We take the role of an online/sequential forecaster producing predictions that will be consumed by agents. We consider the following repeated interaction between a forecaster, agents, and an adversary. In every round $t \in [T]$:
\begin{enumerate}[(1)]
    \item The adversary selects a feature vector $x_t \in \cX$ and a distribution over outcomes $Y_t \in \Delta \cY$;
    \item The forecaster observes the feature $x_t$, produces a distribution over predictions $\pi_t\in\Delta\cY$, from which a prediction $p_t \in \cY$ is sampled; 
    \item Each agent chooses an action $a_t$ as a function of the prediction $p_t$ and the history;
    \item The adversary reveals an outcome $y_t \sim Y_t$, and the agent obtains utility $u(a_t, y_t)$ and the constraint loss vector $\{c_j(a_t, y_t)\}_{j \in [J]}$.
\end{enumerate}

Agents aim to maximize their cumulative utilities over the $T$ rounds: $\sum_{t=1}^T u(a_t,y_t)$
while minimizing their cumulative constraint violation (CCV) over $T$ rounds:
\[
    \CCV(1:T) \coloneqq \max_{j \in [J]} \sum_{t=1}^T c_j(a_t,y_t) \le o(T).
\]
We treat utility maximization as an objective and cumulative constraint violation as a requirement: CCV must be sublinear in $T$ for every constraint. 

We compare the utility of actions chosen based on our predictions against the utility of actions from different benchmark classes. A standard benchmark class from the literature on learning with long term constraints is the set of actions that satisfy the \emph{realized} constraints at every round, denoted as:
\begin{align*}
    \cA_{1:T}^\bc := \left\{ a: c_j(a,y_t) \le 0 \; \text{ for every } t \in [T] \text{ and } j \in [J] \right\}.
\end{align*}

A somewhat more permissive benchmark class is the set of actions that satisfy the constraints \emph{in expectation} every round, denoted as:
\begin{align*}
    \cA_{1:T}^{\E[\bc]} := \left\{ a: \E_{y_t \sim Y_t}[ c_j(a,y_t)] \le 0 \; \text{ for every } t \in [T] \text{ and } j \in [J] \right\}.
\end{align*}

We note that these benchmark classes are agent-specific as they depend on agents' constraint functions. 
For a given benchmark class, we will compete with the best fixed action in hindsight from the class. We measure our performance using a notion of regret, which we call \emph{constrained external regret}.
\begin{definition}[$\cB$-Constrained External Regret] \label{def:external-regret}
    Fix an agent with action space $\cA$ and utility function $u$. Let $\cB \subseteq \cA$ be a benchmark class of actions. For a sequence of actions $a_1,\ldots,a_T$ and outcomes $y_1,\ldots,y_T$, the agent's $\cB$-constrained external regret against the benchmark class $\cB$ is:
    \[
        \Reg_\ext(u,\cB,1:T) := \max_{a \in \cB}\sum_{t=1}^T \left( u(a,y_t) - u(a_t,y_t) \right).
    \]
\end{definition}

We will also compete with a more demanding benchmark based on action modification rules. For a given benchmark class $\cB$, an action modification rule is any function $\phi: \cA \to \cB$ that consistently maps an agent's actions to alternatives within the benchmark class. Competing with the best such rule in hindsight leads to the notion of \emph{constrained swap regret}. This is a stronger notion than constrained external regret, as constrained external regret can be viewed as a special case where the action modification rule is restricted to being a constant function.
\begin{definition}[$\cB$-Constrained Swap Regret] \label{def:swap-regret} 
    Fix an agent with action space $\cA$ and utility function $u$. Let $\cB \subseteq \cA$ be a benchmark class of actions. For a sequence of actions $a_1,\ldots,a_T$ and outcomes $y_1,\ldots,y_T$, the agent's $\cB$-constrained swap regret against the benchmark class $\cB$ is:
    \[
        \Reg_\swap(u,\cB,1:T) := \max_{\phi: \cA \to \cB}\sum_{t=1}^T \left( u(\phi(a_t),y_t) - u(a_t,y_t) \right).
    \]
\end{definition}

All guarantees we provide in this paper, such as regret bounds, will be stated under the assumption that the corresponding benchmark class is non-empty.

\section{Elimination-Based Algorithms and Cumulative Constraint Violation Bounds} \label{sec:elimination-CCV}
In the literature on unconstrained omniprediction, the typical decision rule is the best response to the predictions, also known as the optimal post-processing of the predictions \citep{gopalan2021omnipredictors,gopalan2022loss,gopalan2023swap,globus2023multicalibration,garg2024oracle,gopalan2024omnipredictors,okoroafor2025near,hu2025omnipredicting,lu2025sample}. This decision rule treats the predictions as accurate and chooses the action to maximize the agent's utility.
\begin{definition}[Best Response] \label{def:best-response}
    Fix a utility function $u: \cA \times \cY \to [0,1]$ and a prediction $p \in \cY$. The best response to $p$ according to $u$ is the action $\BR^{u}(p) = \argmax_{a \in \cA} u(a,p)$.
\end{definition}

In the constrained setting, one would hope to avoid actions that (significantly) violate the constraints, and choose only from the feasible actions. If the agent knew the benchmark class $\cA_{1:T}^\bc$, they could play a constrained version of the best response, $\argmax_{a \in \cA^\bc_{1:T}} u(a,p)$, i.e., the action that maximizes the predicted utility among $\cA_{1:T}^\bc$. But the Agents do not know what this benchmark class is during the interaction (it isn't defined until the end of the interaction), and so we let each agent maintain a candidate action set $\widehat\cA_{t}^\bc$, which is updated on every round $t \in [T]$ and serves as a surrogate for the benchmark class $\cA_{1:T}^\bc$. The agent will play $\argmax_{a \in \widehat\cA^\bc_{t}} u(a,p)$, the action that maximizes the predicted utility among their maintained candidate action set $\widehat\cA_{t}^\bc$. Similarly, when competing with the other benchmark class $\cA_{1:T}^{\E[\bc]}$, agents will maintain a surrogate $\widehat\cA_{t}^{\E[\bc]}$, and maximize the predicted utility over actions in $\widehat\cA_{t}^{\E[\bc]}$. In general, we say that agents \emph{constrained best respond} to our predictions.
\begin{definition}[$\cC$-Constrained Best Response] \label{def:set-constrained-best-response}
    Let $\cC \subseteq \cA$ be an action set. Fix a utility $u: \cA \times \cY \to [0,1]$ and a prediction $p \in \cY$. The $\cC$-constrained best response to $p$ according to $u$ is the action $\CBR^{u}_\cC(p) = \argmax_{a \in \cC} u(a,p)$.
\end{definition}

It is straightforward for an agent to identify non-members of the benchmark class $\cA_{1:T}^\bc$ and maintain the candidate action set $\widehat\cA_{t}^\bc$. At the beginning, there is no information about the benchmark class $\cA_{1:T}^\bc$ and the agent can set $\widehat\cA_{1}^\bc = \cA$. 
% In every round $t \in [T]$, the agent chooses an action $a_t$ by $\widehat\cA_{t}^\bc$-constrained best responding to the prediction $p_t$ and then observes the constraint realizations $(c_j(a_t,y_t))_{j \in [J]}$ as the label $y_t$ is revealed. If $a_t$ violates any of the $J$ constraints, it is known not to be in the benchmark class $\cA_{1:T}^\bc$, and hence can be permanently eliminated from the candidate action set $\widehat\cA_{t}^\bc$. 
In every round $t \in [T]$, if any action $a \in \cA$ violates any of the $J$ constraints, $a$ is known not to be in the benchmark class $\cA_{1:T}^\bc$, and hence can be permanently eliminated from the candidate action set $\widehat\cA_{t}^\bc$. 
This process is formally stated in Algorithm \ref{alg:elimination-realization}.

\begin{algorithm}[H]
    % \KwIn{}
    % \KwOut{} 
    % \vspace{.5em}

    Initialize $\widehat\cA^\bc_{1} = \cA$\;

    \For{$t=1$ \KwTo $T$}{
        Receive the prediction $p_t$\;

        Play the $\widehat\cA^\bc_{t}$-constrained best response to $p_t$, $a_t = \CBR^{u}_{\widehat\cA^\bc_{t}}(p_t)$\;

        Observe the outcome $y_t$, receive utility $u(a_t, y_t)$ and observe constraint realizations $c_j(a_t, y_t)$, $j \in [J]$\;
        
        % \If{$\exists j \in [J], c_j(a_t,y_t)>0$}{
        %     $\widehat\cA^\bc_{t+1} = \widehat\cA^\bc_{t} \setminus \{a_t\}$\;
        % } 
        % \Else{
        %     $\widehat\cA^\bc_{t+1} = \widehat\cA^\bc_{t}$\;
        % }
        Let $\widehat\cA^\bc_{t+1} = \widehat\cA^\bc_{t} \setminus \left\{ a \in \widehat\cA^\bc_{t} : \exists j \in [J], c_j(a,y_t)>0 \right\};$
    }    
    \caption{Elimination-Based Algorithm for Benchmark Class $\cA^\bc_{1:T}$}
    \label{alg:elimination-realization}
\end{algorithm}

By design, the elimination operation in Algorithm \ref{alg:elimination-realization} is only triggered by actions that violate a constraint. Since actions in the benchmark class $\cA^\bc_{1:T}$ satisfy the constraints every round, they will never be eliminated from the candidate action set $\widehat\cA^\bc_{t}$ by Algorithm \ref{alg:elimination-realization}.
\begin{proposition} \label{prop:benchmark-realization}
    Fix an agent in $\cN$ with a utility function $u: \cA \times \cY \to [0,1]$ and $J$ constraint functions $\{c_j: \cA \times \cY \to [-1,1]\}_{j \in [J]}$. Suppose the agent runs Algorithm \ref{alg:elimination-realization} to compete with the benchmark class $\cA^\bc_{1:T}$, then $\cA^\bc_{1:T} \subseteq \widehat\cA^\bc_{t}$ for any $t \in [T]$.
\end{proposition}

Because an action is permanently eliminated upon violating the constraints, each constraint violation of Algorithm \ref{alg:elimination-realization} must be triggered by a distinct action. Consequently, the cumulative constraint violation of Algorithm \ref{alg:elimination-realization} can be upper bounded by the number of actions.
\begin{theorem} \label{thm:CCV-realization}
    Fix an agent in $\cN$ with a utility function $u: \cA \times \cY \to [0,1]$ and $J$ constraint functions $\{c_j: \cA \times \cY \to [-1,1]\}_{j \in [J]}$. Suppose the agent runs Algorithm \ref{alg:elimination-realization} to compete with the benchmark class $\cA^\bc_{1:T}$, then the cumulative constraint violation of the agent is bounded by:
    \begin{align*}
        \CCV(1:T) \le |\cA|.
    \end{align*}
\end{theorem}
\begin{proof}
    Fix any $j \in [J]$. We will bound the cumulative violation against the $j$-th constraint, i.e., $\sum_{t=1}^T c_j(a_t, y_t)$. The final result follows by taking the maximum over all $j \in [J]$.

    First, we partition the sum based on the action played. For each action $a \in \cA$, let $R_a = \{t \in [T]: a_t=a\}$ be the set of rounds when $a$ is played. This gives:
    \begin{align*}
        \sum_{t=1}^T c_j(a_t,y_t) &= \sum_{a \in \cA: R_a \neq \emptyset} \sum_{t \in R_a} c_j(a,y_t).
    \end{align*}

    $\max R_a$ corresponds to the last round when $a$ is played. For any earlier round $t \in R_a \setminus \{\max R_a\}$, action $a$ is played and is not eliminated, which implies that the constraint is satisfied, i.e., $c_j(a,y_t) \le 0$. 
    Therefore, only the final term, $c_j(a,y_{\max R_a})$, can contribute positively to the sum for each action.
    The sum is thus bounded by:
    \begin{align*}
        \MoveEqLeft \sum_{a \in \cA: R_a \neq \emptyset} \left( \sum_{t \in R_a \setminus \{\max R_a\}} c_j(a,y_t) + c_j(a,y_{\max R_a}) \right) \\
        &\le \sum_{a \in \cA: R_a \neq \emptyset} c_j(a,y_{\max R_a}) \\
        &\le \sum_{a \in \cA: R_a \neq \emptyset} 1 \\
        &\le |\cA|.
    \end{align*}
\end{proof}

\begin{remark}
    In fact, our proof of Theorem \ref{thm:CCV-realization} leads to a stronger result, $\CCV^+(1:T) \coloneqq \max_{j \in [J]}\sum_{t=1}^T c_j^+(a_t,y_t) \le |\cA| - |\cA^\bc_{1:T}|$, where $c_j^+(a_t,y_t) \coloneqq \max\{0, c_j(a_t,y_t)\}$. The right-hand side can be strengthened because actions in the benchmark class $\cA^\bc_{1:T}$ never trigger any constraint violation. The left-hand side can be strengthened because our proof bounds any non-positive term $c_j(a,y_t)$ by 0. $\CCV^+$ is a common definition of cumulative constraint violation in the literature on learning with long-term constraints (see e.g. \cite{sinha2024optimal,lekeufack2024optimistic}).
\end{remark}

Next, we introduce our elimination algorithm for the benchmark class $\cA^{\E[\bc]}_{1:T}$. The identification of non-members is straightforward for the benchmark class $\cA^\bc_{1:T}$ but not for $\cA^{\E[\bc]}_{1:T}$. Since $c_j(a_t,y_t)$ and $\E_{y_t \sim Y_t}[c_j(a_t,y_t)]$ can have different signs, it is impossible to decide whether an action is in $\cA^{\E[\bc]}_{1:T}$ based on its performance in a single round. Therefore, we will let agents eliminate actions by assessing their performances over multiple rounds. To be specific, for each action $a \in \cA$ and each constraint $c_j$, we let $v^{a,j}_{1:t}$ denote the total violation of $c_j$ accumulated by the agent from playing action $a$ up to round $t$, i.e.,
\begin{align*}
    v^{a,j}_{1:t} \coloneqq \sum_{s=1}^t \1[a_s=a]c_j(a,y_s).
\end{align*}
As we will show in the proof of Theorem \ref{prop:benchmark-expectation}, $v_{1:t}^{a,j}$ will concentrate around its `expectation version', $\sum_{s=1}^t \1[a_s=a] \E_{y_s \sim Y_s}[c_j(a,y_s)]$, which is non-positive for any benchmark action $a \in \cA_{1:T}^{\E[\bc]}$ and is potentially positive for any non-benchmark action $a \notin  \cA_{1:T}^{\E[\bc]}$. So, for each action $a \in \cA$, we will let agents monitor $(v^{a,j}_{1:t})_{j \in [J]}$ in every round and eliminate action $a$ once any of $(v^{a,j}_{1:t})_{j \in [J]}$ exceeds a prefixed threshold $\tau$. This process is formally stated in Algorithm \ref{alg:elimination-expectation}. We will give guarantees for Algorithm \ref{alg:elimination-expectation} that hold with high probability, where the probability is determined by $\tau$.

\begin{algorithm}[H]
    \KwIn{A threshold parameter $\tau$}
    % \KwOut{} 
    % \vspace{.5em}

    Initialize $\widehat\cA^{\E[\bc]}_{1} = \cA$\;

    For every $a \in \cA$ and $j \in [J]$, set $v^{a,j}_0 = 0$\;

    \For{$t=1$ \KwTo $T$}{
        Receive the prediction $p_t$\;

        Play the $\widehat\cA^{\E[\bc]}_{t}$-constrained best response to $p_t$, $a_t = \CBR^{u}_{\widehat\cA^{\E[\bc]}_{t}}(p_t)$\;

        Observe the outcome $y_t$, obtain utility $u(a_t, y_t)$ and constraint increments $c_j(a_t, y_t)$, $j \in [J]$\;

        Update $v^{a_t,j}_{1:t} = v^{a_t,j}_{1:t-1} + c_j(a_t,y_t)$ for every $j \in [J]$\;

        Let $v^{a,j}_{1:t} = v^{a,j}_{1:t-1}$ for every $a \in \cA \setminus \{a_t\}$ and $j \in [J]$\;
        
        \If{$\exists j \in [J], v^{a_t,j}_t > \tau$}{
            $\widehat\cA^{\E[\bc]}_{t+1} = \widehat\cA^{\E[\bc]}_{t} \setminus \{a_t\}$\;
        } 
        \Else{
            $\widehat\cA^{\E[\bc]}_{t+1} = \widehat\cA^{\E[\bc]}_{t}$\;
        }
    }
    
    \caption{Elimination-Based Algorithm for Benchmark Class $\cA^{\E[\bc]}_{1:T}$}
    \label{alg:elimination-expectation}
\end{algorithm}

Algorithm \ref{alg:elimination-expectation} never eliminates any benchmark action in $\cA^{\E[\bc]}_{1:T}$ from the candidate action set $\cA^{\E[\bc]}_t$ across the $T$ rounds with high probability over the randomness of the outcomes $y_1,\ldots,y_T$.
\begin{proposition} \label{prop:benchmark-expectation}
    Fix an agent in $\cN$ with a utility function $u: \cA \times \cY \to [0,1]$ and $J$ constraint functions $\{c_j: \cA \times \cY \to [-1,1]\}_{j \in [J]}$. Suppose the agent runs Algorithm \ref{alg:elimination-expectation} with threshold parameter $\tau$ to compete with the benchmark class $\cA^{\E[\bc]}_{1:T}$, then with probability at least $1-2|\cA|JTe^{-\frac{\tau^2}{8T}}$:
    \begin{align*}
        \cA_{1:T}^{\E[\bc]} \subseteq \widehat\cA_t^{\E[\bc]} \text{ for every } t \in [T].
    \end{align*}
\end{proposition}
\begin{proof}
Fix any $a \in \cA^{\E[\bc]}_{1:T}$ and $j \in [J]$. By the definition of the benchmark class $\cA^{\E[\bc]}_{1:T}$, we have that for every $t \in [T]$:
    \begin{align*}
        \Pr[v^{a,j}_{1:t} > \tau] &= \Pr\left[ \sum_{s=1}^t \1[a_s=a]c_j(a,y_s) > \tau \right] \\
        &\le \Pr\left[ \sum_{s=1}^t \1[a_s=a] (c_j(a,y_s)-\E_{y_s \sim Y_s}[c_j(a,y_s)]) > \tau \right].
    \end{align*}

    We will adopt a martingale concentration inequality to bound the above probability. For this purpose, we define the filtration $\cF_t = \sigma\left(\{(x_s,p_s,y_s)\}_{s=1}^t, x_{t+1}, p_{t+1}\right)$. 
    Intuitively, $\cF_t$ represents the complete history of information available to the agent before making their decision in round $t+1$, which is why $x_{t+1}$ and $p_{t+1}$ are included. We first show that $\{\1[a_s=a] (c_j(a,y_s)-\E_{y_s \sim Y_s}[c_j(a,y_s)])\}_{s=1}^T$ is a sequence of martingale differences with respect to this filtration.
    
   % We note that the forecaster produces the prediction $p_t$ in two steps: After round $t-1$, the forecaster produces an implicit prediction function $\tilde p_t: \cX \to \cY$. Then the forecaster observes the feature $x_t$ and predicts $\tilde p_t(x_t)$.

   % As a result, the action played by the agent in round $t$ can be written as $a_t = \CBR^u_{\widehat\cA_t^{\E[\bc]}}(\tilde p_t(x_t))$. $\widehat\cA_t^{\E[\bc]}$ is the candidate action set maintained by the agent based on the record of the actions up to round $t-1$, so $\widehat\cA_t^{\E[\bc]} \in \cF_{t-1}$. $\tilde p_t$ is produced by the forecaster using the information up to round $t-1$, so $\tilde p_t \in \cF_{t-1}$. We define $g_{t-1}(\cdot) = \CBR^u_{\widehat\cA_t^{\E[\bc]}}(\tilde p_t(\cdot))$ for convenience. Then $a_t = g_{t-1}(x_t)$ and $g_{t-1} \in \cF_{t-1}$.

   We have that $a_t = \CBR^u_{\widehat\cA_t^{\E[\bc]}}(p_t)$. $\widehat\cA_t^{\E[\bc]}$ is the candidate action set maintained by the agent based on the record of the actions up to round $t-1$, so $\widehat\cA_t^{\E[\bc]} \in \cF_{t-1}$. $p_t \in \cF_{t-1}$ by the definition of the filtration. Therefore, $a_t \in \cF_{t-1}$.

   Consequently, for any $t \in [T]$: 
   % \begin{align*}
   %     \E\left[ \1[a_t=a] (c_j(a,y_t)-\E_{y_t \sim Y_t}[c_j(a,y_t)]) \mid \cF_{t-1} \right] &= \E\left[ \1[g_{t-1}(x_t)=a] (c_j(a,y_t)-\E_{y_t \sim Y_t}[c_j(a,y_t)]) \mid \cF_{t-1} \right] = 0
   % \end{align*}
   \begin{align*}
       \MoveEqLeft \E\left[ \1[a_t=a] (c_j(a,y_t)-\E_{y_t \sim Y_t}[c_j(a,y_t)]) \mid \cF_{t-1} \right] \\
       &= \1[a_t=a] \E\left[ c_j(a,y_t)-\E_{y_t \sim Y_t}[c_j(a,y_t)] \mid \cF_{t-1} \right] \\
       &= 0.
   \end{align*}
    This implies that $\{\1[a_s=a] (c_j(a,y_s)-\E_{y_s \sim Y_s}[c_j(a,y_s)])\}_{s=1}^t$ is a sequence of martingale differences. By Azuma-Hoeffding inequality (Lemma \ref{lem:azuma}), we derive that for any $a \in \cA^{\E[\bc]}_{1:T}$, $j \in [J]$, $t \in [T]$:
    \begin{align*}
        \Pr[v^{a,j}_{1:t} > \tau] &\le \Pr\left[ \sum_{s=1}^t \1[a_s=a] (c_j(a,y_s)-\E_{y_s \sim Y_s}[c_j(a,y_s)]) > \tau \right] \\
        &\le 2e^{-\frac{\tau^2}{8t}} \\
        &\le 2e^{-\frac{\tau^2}{8T}}.
    \end{align*}

    Using the union bound over every $a \in \cA^{\E[\bc]}_{1:T}$, $j \in [J]$, $t \in [T]$, we derive that:
    \begin{align*}
        \Pr\left[ \exists a \in \cA^{\E[\bc]}_{1:T}, j \in [J], t \in [T], \text{ such that } v^{a,j}_{1:t} > \tau \right] &\le \sum_{a \in \cA^{\E[\bc]}_{1:T}}\sum_{j \in [J]}\sum_{t \in [T]} 2e^{-\frac{\tau^2}{8T}} \\
        &\le 2|\cA|JTe^{-\frac{\tau^2}{8T}}.
    \end{align*}

    By the elimination rule of Algorithm \ref{alg:elimination-expectation}, this implies that the agent never eliminates any action in $\cA^{\E[\bc]}_{1:T}$ across the $T$ rounds with probability at least $1-2|\cA|JTe^{-\frac{\tau^2}{8T}}$.
\end{proof}

% Proposition \ref{prop:benchmark-expectation} is stated with respect to a single agent. Using the union bound over every agent in $\cN$, we derive that no agents in $\cN$ eliminates any benchmark action in $\cA^{\E[\bc]}_{1:T}$ across the $T$ rounds with probability at least $1-2|\cA||\cN|JTe^{-\frac{\tau^2}{8T}}$. Therefore, we will set $\tau = 4\sqrt{T\ln\frac{|\cA||\cN|JT}{\delta}}$ where $\delta \in (0,1)$ is a (presumably small) constant. This leads to the following corollary regarding every agent, which holds with probability at least $1-\delta$.
Proposition \ref{prop:benchmark-expectation} provides a guarantee for a single agent. To extend this to all agents in $\cN$, we apply a union bound. By setting the threshold $\tau = 4\sqrt{T\ln\frac{|\cA||\cN|JT}{\delta}}$ for a desired failure probability $\delta \in (0,1)$, we arrive at the following corollary, which holds with probability at least $1-\delta$.
\begin{corollary} \label{cor:benchmark-expectation}
    Let $\cN$ be a set of agents, where each agent is equipped with a utility function $u: \cA \times \cY \to [0,1]$ and $J$ constraint functions $\{c_j: \cA \times \cY \to [-1,1]\}_{j \in [J]}$. Each agent runs Algorithm \ref{alg:elimination-expectation} with threshold parameter $4\sqrt{T\ln\frac{|\cA||\cN|JT}{\delta}}$ to compete with the benchmark class $\cA^{\E[\bc]}_{1:T}$, then with probability at least $1-\delta$, for every agent:
    \begin{align*}
        \cA_{1:T}^{\E[\bc]} \subseteq \widehat\cA_t^{\E[\bc]} \text{ for every } t \in [T].
    \end{align*}
\end{corollary}

Suppose an action is eliminated by Algorithm \ref{alg:elimination-expectation} in round $t$. Its contribution to the cumulative constraint violation of the agent must be smaller than $\tau$ before round $t$, otherwise it would be eliminated before round $t$. In round $t$, it violates the constraints by at most 1. Therefore, the agent accumulates at most $\tau+1$ constraint violation when playing a specific action. This fact leads to a bound on the cumulative constraint violation of Algorithm \ref{alg:elimination-expectation}.
\begin{theorem} \label{thm:CCV-expectation}
    Let $\cN$ be a set of agents, where each agent is equipped with a utility function $u: \cA \times \cY \to [0,1]$ and $J$ constraint functions $\{c_j: \cA \times \cY \to [-1,1]\}_{j \in [J]}$. Each agent runs Algorithm \ref{alg:elimination-expectation} with threshold parameter $4\sqrt{T\ln\frac{|\cA||\cN|JT}{\delta}}$ to compete with the benchmark class $\cA^{\E[\bc]}_{1:T}$, then with probability at least $1-\delta$, the cumulative constraint violation of every agent is bounded by:
    \begin{align*}
        \CCV(1:T) \le 4 |\cA| \sqrt{T\ln\frac{|\cA||\cN|JT}{\delta}} + |\cA|.
    \end{align*}
\end{theorem}
\begin{proof}
    By Corollary \ref{cor:benchmark-expectation}, no agent in $\cN$ eliminates any benchmark action in $\cA_{1:T}^{\E[\bc]}$ across the $T$ rounds with probability at least $1-\delta$. We will condition on this event, which ensures that the Algorithm \ref{alg:elimination-expectation} is well-behaved and avoids the risk of eliminating all actions in $\cA$.

    Fix any $j \in [J]$. We will bound the cumulative violation against the $j$-th constraint, i.e., $\sum_{t=1}^T c_j(a_t, y_t)$. The final result follows by taking the maximum over all $j \in [J]$.
    
    Similarly to the proof of Theorem \ref{thm:CCV-realization}, we let $R_a$ denote the set of rounds when $a$ is played, i.e.,
    \begin{align*}
        R_a = \{t: a_t=a\}
    \end{align*}
    
    Then we have that:
    \begin{align*}
        \sum_{t=1}^T c_j(a_t,y_t) &= \sum_{a \in \cA: R_a \neq \emptyset} \sum_{t \in R_a} c_j(a,y_t) \\
        &= \sum_{a \in \cA: R_a \neq \emptyset} \sum_{t=1}^{\max R_a} \1[t \in R_a] c_j(a,y_t) \\
        &= \sum_{a \in \cA: R_a \neq \emptyset} \sum_{t=1}^{\max R_a} \1[a_t=a] c_j(a,y_t) \\
        &= \sum_{a \in \cA: R_a \neq \emptyset} \left( \sum_{t=1}^{\max R_a-1} \1[a_t=a] c_j(a,y_t) + c_j(a,y_{\max R_a}) \right) \\
        &= \sum_{a \in \cA: R_a \neq \emptyset} (v^{a,j}_{\max R_a-1} + c_j(a,y_{\max R_a})) \\
        &\le \sum_{a \in \cA: R_a \neq \emptyset} (\tau + 1) \\
        &\le 4 |\cA| \sqrt{T\ln\frac{|\cA||\cN|JT}{\delta}} + |\cA|.
    \end{align*}
\end{proof}

We remark that both Algorithms \ref{alg:elimination-realization} and \ref{alg:elimination-expectation} are run by agents. While they take the prediction $p_t$ as input in every round, their cumulative constraint violation bounds in Theorems \ref{thm:CCV-realization} and \ref{thm:CCV-expectation} are proven without any assumptions on the predictions. These bounds are consequences of the elimination rules and hold true for any arbitrary forecasters.

\section{Conditionally Unbiased Predictions and Regret Bounds} \label{sec:regret}

In this section, we show how to produce predictions that lead to $\tilde O(|\cA|\sqrt{T})$ regret against the benchmark class $\cA^\bc_{1:T}$. Guarantees for the benchmark class $\cA^{\E[\bc]}_{1:T}$ with the same dependence on $|\cA|$ and $T$ can be derived similarly and are relegated to Appendix \ref{app:regret}.

While the cumulative constraint violation bounds of Algorithms \ref{alg:elimination-realization} and \ref{alg:elimination-expectation} do not rely on the quality of the predictions, to obtain the regret bounds, we need the predictions to be unbiased --- 
% not only marginally, but also 
conditionally on various subsequences. Notably, we define the following notion of conditional unbiasedness.

\begin{definition}[$(\cN,\alpha_1)$-Decision Calibration] \label{def:decision-calibration-realization}
    Let $\cN$ be a set of agents, where each agent is equipped with a utility function $u: \cA \times \cY \to [0,1]$ and $J$ constraint functions $\{c_j: \cA \times \cY \to [-1,1]\}_{j \in [J]}$. Each agent runs Algorithm \ref{alg:elimination-realization} to maintain a candidate action set $\widehat\cA_t^\bc$ and $\widehat\cA_t^\bc$-constrained best respond to the prediction $p_t$ in every round $t \in [T]$. Let $\alpha_1: \bR \to \bR$. We say that a sequence of predictions $p_1,\ldots,p_T$ is $(\cN,\alpha_1)$-decision calibrated with respect to a sequence of outcomes $y_1,\ldots,y_T$ if for every $a\in \cA$ and $(u,\bc) \in \cN$:
    \[
        \left\| \sum_{t=1}^T \1[\CBR^u_{\widehat\cA_{t}^\bc}(p_t)=a] (p_t - y_t) \right\|_\infty \leq \alpha_1(T^{u,\bc}(a)).
    \]
    where $T^{u,\bc}(a) = \sum_{t=1}^T \1[\CBR^u_{\widehat\cA_t^\bc}(p_t)=a]$.
\end{definition}

\begin{assumption}
    We assume that $\alpha_1: \bR \to \bR$ is a concave function. This will be the case in all the bounds we give; in general, this condition holds for any sublinear error bound $T^r$ for $r<1$.
\end{assumption}

We now show that decision calibration implies no constrained swap regret, and hence no constrained external regret. At a high level, decision calibrated predictions allow agents to accurately assess (on average) the utilities of their chosen actions and the counterfactual actions produced by any action modification rule. Because agents choose optimally based on these accurate utility estimates, their decisions are inherently competitive against any alternative benchmark.

\begin{theorem} \label{thm:swap-regret-realization}
    Let $\cN$ be a set of agents, where each agent is equipped with a utility function $u: \cA \times \cY \to [0,1]$ and $J$ constraint functions $\{c_j: \cA \times \cY \to [-1,1]\}_{j \in [J]}$. The utility functions are linear and $L$-Lipschitz in the second argument. Each agent runs Algorithm \ref{alg:elimination-realization} to compete with the benchmark class $\cA^{\bc}_{1:T}$.
    If the sequence of predictions $p_1,\ldots,p_T$ is $(\cN,\alpha_1)$-decision calibrated, then the $\cA^\bc_{1:T}$-constrained swap regret of every agent is bounded by:
    \begin{align*}
        \Reg_\swap(u,\cA_{1:T}^\bc,1:T) \le 2 L |\cA| \alpha_1(T/|\cA|).
    \end{align*}
\end{theorem}

To prove the theorem, first note that for any action modification rule $\phi: \cA \to \cA_{1:T}^\bc$, we can decompose the regret against $\phi$ into three parts as:
\begin{align*}
    \MoveEqLeft \sum_{t=1}^T \left( u(\phi(a_t),y_t) - u(a_t,y_t) \right) \\
    &= \sum_{t=1}^T \left( u(\phi(a_t),y_t) - u(\phi(a_t),p_t) \right) + \sum_{t=1}^T \left( u(\phi(a_t),p_t) - u(a_t,p_t) \right) + \sum_{t=1}^T \left( u(a_t,p_t) - u(a_t,y_t) \right)
\end{align*}

By Proposition \ref{prop:benchmark-realization} and the fact that $\phi(a_t) \in \cA_{1:T}^\bc$, we have that $\phi(a_t) \in \widehat\cA_t^\bc$ for every $t \in [T]$. Since agents $\widehat\cA_t^\bc$-constrained best response to the prediction $p_t$ in round $t$, $a_t = \argmax_{b \in \widehat\cA_t^\bc} u(b,p_t)$. This implies that:
\begin{align*}
    u(\phi(a_t),p_t) \le u(a_t,p_t).
\end{align*}
So, the second part in the decomposition is upper bounded by 0.

Thus, it suffices to bound the other two parts, i.e., the difference in utility under our predictions $p_t$ and the outcomes $y_t$ for both the chosen actions and the swapped-in actions. We show this in the next two lemmas using decision calibration. 

\begin{lemma} \label{lem:decision-calibration-realization}
    If the sequence of predictions $p_1,\ldots,p_T$ is $(\cN,\alpha_1)$-decision calibrated, then for any $(u,\bc) \in \cN$:
    \begin{align*}
        \left| \sum_{t=1}^T (u(a_t,p_t) - u(a_t,y_t)) \right| \le L |\cA| \alpha_1(T/|\cA|).
    \end{align*}
\end{lemma}
\begin{proof}
    Using the linearity of $u$, we can write:
    \begin{align*}
        \left| \sum_{t=1}^T (u(a_t,p_t) - u(a_t,y_t)) \right| &= \left| \sum_{a \in \cA}\sum_{t=1}^T \1[a_t=a] (u(a,p_t) - u(a,y_t)) \right| \\
        &= \left| \sum_{a \in \cA} \left( u\left(a, \sum_{t=1}^T \1[a_t=a] p_t \right) - u\left(a, \sum_{t=1}^T \1[a_t=a] y_t\right) \right)\right| \\
        &\le \sum_{a \in \cA} \left| u\left(a, \sum_{t=1}^T \1[a_t=a] p_t \right) - u\left(a, \sum_{t=1}^T \1[a_t=a] y_t\right) \right| \\
        &\le \sum_{a \in \cA} L \left\| \sum_{t=1}^T \1[a_t=a] (p_t-y_t) \right\|_\infty \\
        &\le L \sum_{a \in \cA} \alpha_1(T^{u,\bc}(a)).
    \end{align*}
    where the first inequality follows from the triangle inequality, the second inequality follows from $L$-Lipschitzness of $u$, and the third inequality follows from $(\cN,\alpha_1)$-decision calibration. By concavity of $\alpha_1$ and the fact that $\sum_{a \in \cA} T^{u,\bc}(a) = \sum_{a \in \cA} \sum_{t=1}^T \1[a_t=a] = T$, this expression is at most:
    \begin{align*}
        L |\cA| \alpha_1(T/|\cA|).
    \end{align*}
\end{proof}

\begin{lemma} \label{lem:decision-calibration-realization-swap}
    If the sequence of predictions $p_1,\ldots,p_T$ is $(\cN,\alpha_1)$-decision calibrated, then for any $(u,\bc) \in \cN$ and $\phi: \cA \to \cA_{1:T}^\bc$:
    \begin{align*}
        \left| \sum_{t=1}^T (u(\phi(a_t),p_t) - u(\phi(a_t),y_t)) \right| \le L |\cA| \alpha_1(T/|\cA|).
    \end{align*}
\end{lemma}
\begin{proof}
    The proof is similar to the proof of Lemma \ref{lem:decision-calibration-realization}. Using the linearity of $u$, we can write:
    \begin{align*}
        \left| \sum_{t=1}^T (u(\phi(a_t),p_t) - u(\phi(a_t),y_t)) \right| &= \left| \sum_{a \in \cA}\sum_{t=1}^T \1[a_t=a] (u(\phi(a),p_t) - u(\phi(a),y_t)) \right| \\
        &= \left| \sum_{a \in \cA} \left( u\left(\phi(a), \sum_{t=1}^T \1[a_t=a] p_t \right) - u\left(\phi(a), \sum_{t=1}^T \1[a_t=a] y_t\right) \right)\right| \\
        &\le \sum_{a \in \cA} \left| u\left(\phi(a), \sum_{t=1}^T \1[a_t=a] p_t \right) - u\left(\phi(a), \sum_{t=1}^T \1[a_t=a] y_t\right) \right| \\
        &\le \sum_{a \in \cA} L \left\| \sum_{t=1}^T \1[a_t=a] (p_t-y_t) \right\|_\infty \\
        &\le L \sum_{a \in \cA} \alpha_1(T^{u,\bc}(a)) \\
        &\le L |\cA| \alpha_1(T/|\cA|).
    \end{align*}
    where the first inequality follows from the triangle inequality, the second inequality follows from $L$-Lipschitzness of $u$, the third inequality follows from $(\cN,\alpha_1)$-decision calibration, and the fourth inequality follows from concavity of $\alpha_1$.
\end{proof}

We can now complete the proof of Theorem \ref{thm:swap-regret-realization}.
\begin{proof}
    Fix any $(u,\bc) \in \cN$ and $\phi: \cA \to \cA^\bc_{1:T}$. Applying Lemmas \ref{lem:decision-calibration-realization} and \ref{lem:decision-calibration-realization-swap} to the decomposition of the regret against $\phi$, we have that:
    \begin{align*}
        \MoveEqLeft \sum_{t=1}^T \left( u(\phi(a_t),y_t) - u(a_t,y_t) \right) \\
        &= \sum_{t=1}^T \left( u(\phi(a_t),y_t) - u(\phi(a_t),p_t) \right) + \sum_{t=1}^T \left( u(\phi(a_t),p_t) - u(a_t,p_t) \right) + \sum_{t=1}^T \left( u(a_t,p_t) - u(a_t,y_t) \right) \\
        &\le 2 L |\cA| \alpha_1(T/|\cA|).
    \end{align*}
\end{proof}

\subsection{Algorithm for Calibration and Decision Calibration}

We now address the algorithmic challenge of producing decision calibrated predictions. Our approach builds upon the \textsc{Unbiased-Prediction} algorithm from \citet{noarov2023highdimensional}, which makes conditionally unbiased predictions in the online setting. The algorithm and its guarantees are presented in Appendix \ref{app:unbiased-prediction}; we refer interested readers to the original work for further details.

We will instantiate \textsc{Unbiased-Prediction} to make decision calibrated predictions; we will refer to this instantiation as \textsc{Decision-Calibration-Realization}. 
Our guarantees will directly inherit from the guarantees of \textsc{Unbiased-Prediction}. 

\begin{theorem} \label{thm:unbiased-algorithm-realization}
    Let $\cN$ be a set of agents, where each agent is equipped with a utility function $u: \cA \times \cY \to [0,1]$ and $J$ constraint functions $\{c_j: \cA \times \cY \to [-1,1]\}_{j \in [J]}$. Each agent will run Algorithm \ref{alg:elimination-realization} to compete with the benchmark class $\cA^{\bc}_{1:T}$.
    There is an instantiation of \textsc{Unbiased-Prediction} \citep{noarov2023highdimensional}
    ---which we call \textsc{Decision-Calibration-Realization}---
    producing predictions $p_1,...,p_T \in \cY$ satisfying that for any sequence of outcomes $y_1,...,y_T \in \cY$, the following holds with probability at least $1-\delta$ for any
    $(u,\bc) \in \cN$ and $a \in \cA$:
        \[
            \left\| \sum_{t=1}^T \1[\CBR^u_{\widehat\cA_t^\bc}(p_t) = a] (p_t - y_t) \right\|_\infty \leq O\left( \ln(d|\cA||\cN|T) + \sqrt{\ln(d|\cA||\cN|T) \cdot T^{u,\bc}(a)} + \sqrt{\ln\frac{d|\cA||\cN|}{\delta} \cdot T}\right).
        \]
    
\end{theorem}

Substituting the above bounds into Theorem \ref{thm:swap-regret-realization}, we arrive at the following corollary bounding the $\cA^\bc_{1:T}$-constrained swap regret.

\begin{corollary} \label{cor:swap-regret-realization}
    Let $\cN$ be a set of agents, where each agent is equipped with a utility function $u: \cA \times \cY \to [0,1]$ and $J$ constraint functions $\{c_j: \cA \times \cY \to [-1,1]\}_{j \in [J]}$. The utility functions are linear and $L$-Lipschitz in the second argument. Each agent runs Algorithm \ref{alg:elimination-realization} to compete with the benchmark class $\cA^{\bc}_{1:T}$.
    The sequence of predictions $p_1,\ldots,p_T$ outputted by \textsc{Decision-Calibration-Realization} ensures that with probability at least $1-\delta$, the $\cA^\bc_{1:T}$-constrained swap regret of every agent is bounded by:
    \begin{align*}
        \Reg_\swap(u,\cA_{1:T}^\bc,1:T) \le O\left( L |\cA| \sqrt{T\ln\frac{d|\cA||\cN|T}{\delta}} \right).
    \end{align*}
\end{corollary}

Since constrained swap regret is stronger than constrained external regret, we can also bound the $\cA^\bc_{1:T}$-constrained external regret.

\begin{corollary} \label{cor:external-regret-realization}
    Let $\cN$ be a set of agents, where each agent is equipped with a utility function $u: \cA \times \cY \to [0,1]$ and $J$ constraint functions $\{c_j: \cA \times \cY \to [-1,1]\}_{j \in [J]}$. The utility functions are linear and $L$-Lipschitz in the second argument. Each agent runs Algorithm \ref{alg:elimination-realization} to compete with the benchmark class $\cA^{\bc}_{1:T}$.
    The sequence of predictions $p_1,\ldots,p_T$ outputted by \textsc{Decision-Calibration-Realization} ensures that with probability at least $1-\delta$, the $\cA^\bc_{1:T}$-constrained external regret of every agent is bounded by:
    \begin{align*}
        \Reg_\ext(u,\cA_{1:T}^\bc,1:T) \le O\left( L |\cA| \sqrt{T\ln\frac{d|\cA||\cN|T}{\delta}} \right).
    \end{align*}
\end{corollary}

\section{Simultaneous Regret Minimization for Multiple Subsequences} \label{sec:subsequence}

Our analysis so far has focused on achieving low regret against the best single benchmark action in hindsight, evaluated over all T rounds. We now extend this framework to a more demanding and general setting by considering performance over multiple, arbitrary subsequences of rounds. We will compete with different benchmarks over different subsequences, and provide more granular guarantees that hold simultaneously across all subsequences.

Let $\cS$ be a collection of subsequences, where each subsequence $S \in \cS$ is a subset of $[T]$. We assume that the union of these subsequences covers the entire time horizon, i.e., $\cup_{S \in \cS} S = [T]$. This ensures that every round belongs to at least one subsequence, and thus has a benchmark for us to compete with.
We denote the active subsequences in round $t \in [T]$ as $\cS_t = \{S \in \cS: t \in S\}$, which is non-empty by our assumption. These subsequences need not be fixed in advance but can be defined dynamically based on contextual information. A subsequence $S \in \cS$ is generally characterized by an indicator function $h_S: [T] \times \cX \to \{0,1\}$. For any round $t \in [T]$ with an observed feature $x_t$, the round is part of the subsequence $S$ if and only if $h_S(t, x_t) = 1$. This flexible definition allows subsequences to be based on the round index $t$, the feature $x_t$, or both. We could also generalize the definition of subsequence to let it depend arbitrarily on the history.

We now extend our benchmark definitions to the multi-subsequence setting.
Fix a subsequence $S \in \cS$, one benchmark class of interest is the set of actions that satisfy the realized constraints over the entire subsequence, denoted as:
\begin{align*}
    \cA_{S}^\bc = \left\{ a: c_j(a,y_t) \le 0 \; \text{ for every } t \in S \text{ and } j \in [J] \right\}.
\end{align*}

Another benchmark class is the set of actions that satisfy the constraints in expectation over the entire subsequence, denoted as:
\begin{align*}
    \cA_{S}^{\E[\bc]} = \left\{ a: \E_{y_t \sim Y_t}[ c_j(a,y_t)] \le 0 \; \text{ for every } t \in S \text{ and } j \in [J] \right\}.
\end{align*}

Previously, our notions of constrained external regret and constrained swap regret were defined over the entire time horizon of $T$ rounds. We now extend these definitions to handle arbitrary subsequences. Note that the ``best action in hindsight'' can be different on each subsequence, so regret bounds that hold simultaneously on many subsequences are comparing to a substantially richer benchmark class than regret bounds that hold marginally over a single subsequence.
\begin{definition}[$\cB$-Constrained External Regret over Subsequence $S$] 
    Fix an agent with action space $\cA$ and utility function $u$. Let $\cB$ be a benchmark class of actions. Fix a subsequence $S \subseteq [T]$. For a sequence of actions $a_1,\ldots,a_T$ and outcomes $y_1,\ldots,y_T$, the agent's $\cB$-constrained external regret against the benchmark class $\cB$ over the subsequence $S$ is:
    \[
        \Reg_\ext(u,\cB,S) = \max_{a \in \cB}\sum_{t \in S} \left( u(a,y_t) - u(a_t,y_t) \right).
    \]
\end{definition}

\begin{definition}[$\cB$-Constrained Swap Regret over Subsequence $S$] 
    Fix an agent with action space $\cA$ and utility function $u$. Let $\cB$ be a benchmark class of actions. Fix a subsequence $S \subseteq [T]$. For a sequence of actions $a_1,\ldots,a_T$ and outcomes $y_1,\ldots,y_T$, the agent's $\cB$-constrained swap regret against the benchmark class $\cB$ over the subsequence $S$ is:
    \[
        \Reg_\swap(u,\cB,S) = \max_{\phi: \cA \to \cB} \sum_{t \in S} \left( u(\phi(a_t),y_t) - u(a_t,y_t) \right).
    \]
\end{definition}

Similarly, we define the cumulative constraint violation over a subsequence:
\begin{align*}
    \CCV(S) \coloneqq \max_{j \in [J]} \sum_{t \in S} c_j(a_t,y_t).
\end{align*}

Agents aim to achieve low regret and satisfy the long-term constraints simultaneously over every subsequence in $\cS$. The latter requires that $\CCV(S) \le o(|S|)$ for any $S \in \cS$, where $|S|$ is the length of subsequence $S$.

To compete with the benchmark classes $\{\cA_S^\bc\}_{S \in \cS}$, we use Algorithm \ref{alg:elimination-realization} as the fundamental building block for our multi-subsequence setting. We will let each agent instantiate a copy of Algorithm \ref{alg:elimination-realization} for every subsequence. In each round $t \in [T]$, these parallel instantiations provide the agent with a candidate action set $\widehat\cA^\bc_{t,S}$ for every active subsequence $S \in \cS_t$. The agent then aggregates these candidate action sets by taking their union, $U_t = \cup_{S \in \cS_t} \widehat\cA_{t,S}^\bc$, and $U_t$-constrained best respond to our prediction $p_t$. This procedure is formally stated in Algorithm \ref{alg:elimination-realization-subsequence}.

\begin{algorithm}[H]
    % \KwIn{A single-sequence elimination algorithm \textsc{Elimination}}
    % \KwOut{} 
    % \vspace{.5em}

    Initialize $\widehat\cA_{1,S}^\bc = \cA$ for every subsequence $S \in \cS$\;

    \For{$t=1$ \KwTo $T$}{
        Receive the prediction $p_t$\;

        Aggregate candidate action sets from active subsequences, $U_{t} = \cup_{S \in \cS_t} \widehat\cA_{t,S}^\bc$\;
        
        Play the $U_t$-constrained best response to $p_t$, $a_t = \CBR^{u}_{U_t}(p_t)$\;

        Observe the outcome $y_t$, obtain utility $u(a_t, y_t)$ and constraint increments $c_j(a_t, y_t)$, $j \in [J]$\;
        
        Obtain $\widehat\cA_{t+1,S}^\bc$ by executing one step of Algorithm \ref{alg:elimination-realization} for every active subsequence $S \in \cS_t$\;
        
        % Let $\widehat\cA^\bc_{t+1,S} = \widehat\cA^\bc_{t,S} \setminus \left\{ a \in \widehat\cA^\bc_{t,S} : \exists j \in [J], c_j(a,y_t)>0 \right\}$ for every active subsequence $S \in \cS_t$\;
        
        Let $\widehat\cA_{t+1,S}^\bc = \widehat\cA_{t,S}^\bc$ for every non-active subsequence $S \in \cS \setminus \cS_t$\;
    }
    
    \caption{Elimination-Based Algorithm for Benchmark Classes $\{\cA^\bc_S\}_{S \in \cS}$}
    \label{alg:elimination-realization-subsequence}
\end{algorithm}

Since each instantiation of Algorithm \ref{alg:elimination-realization} is guaranteed not to eliminate the benchmark actions corresponding to its subsequence (Proposition \ref{prop:benchmark-realization}), the union set $U_t$ at each round $t \in [T]$ will therefore contain the benchmark actions for every active subsequence $S \in \cS_t$.
\begin{proposition} \label{prop:benchmark-realization-subsequence}
    Let $\cS$ be a collection of subsequences. Fix an agent in $\cN$ with a utility function $u: \cA \times \cY \to [0,1]$ and $J$ constraint functions $\{c_j: \cA \times \cY \to [-1,1]\}_{j \in [J]}$. Suppose the agent runs Algorithm \ref{alg:elimination-realization-subsequence} to compete with the benchmark class $\cA^\bc_S$ over every subsequence $S \in \cS$, then $\cA^\bc_{S} \subseteq \widehat\cA_{t,S}^\bc \subseteq U_{t}$ for any $t \in [T]$ and $S \in \cS_t$.
\end{proposition}

% Recall that in the single-sequence case, the cumulative constraint violation is bounded by $|\cA|$ because each action is eliminated upon violating the constraints, and hence can violate the constraints at most once. 
Fix any subsequence $S \in \cS$. Since the agent chooses actions from the union sets, an action's selection can be sourced from the candidate action set of any subsequence in $\cS$. As a result, each action can contribute at most $|\cS|$ times to the cumulative constraint violation over $S$ before it is fully eliminated.
Consequently, the cumulative constraint violation of Algorithm \ref{alg:elimination-realization-subsequence} over any subsequence can be upper bounded by $|\cA| |\cS|$.
\begin{theorem} \label{thm:CCV-realization-subsequence}
    Let $\cS$ be a collection of subsequences. Fix an agent in $\cN$ with a utility function $u: \cA \times \cY \to [0,1]$ and $J$ constraint functions $\{c_j: \cA \times \cY \to [-1,1]\}_{j \in [J]}$. Suppose the agent runs Algorithm \ref{alg:elimination-realization-subsequence} to compete with the benchmark class $\cA^\bc_S$ over every subsequence $S \in \cS$, then the cumulative constraint violation of the agent over any subsequence $S \in \cS$ is bounded by:
    \begin{align*}
        \CCV(S) \le |\cA| |\cS|.
    \end{align*}
\end{theorem}
\begin{proof}
    Fix any $j \in [J]$ and $S \in \cS$. We will bound the cumulative violation against the $j$-th constraint over $S$, i.e., $\sum_{t \in S} c_j(a_t, y_t)$. The final result follows by taking the maximum over all $j \in [J]$.

    Similarly to the proof of Theorem \ref{thm:CCV-realization}, we first partition the sum based on the action played. For each action $a \in \mathcal{A}$, let $R_{S,a} = \{t: t \in S, a_t=a\}$. This gives:
    \begin{align*}
        \sum_{t \in S} c_j(a_t,y_t) &= \sum_{a \in \cA: R_{S,a} \neq \emptyset} \sum_{t \in R_{S,a}} c_j(a,y_t).
    \end{align*}

    For an action $a$ to be played in round $t$, it must be a candidate action for at least one active subsequence $S'$. As a result, the above equation can be upper bounded by:
    \begin{align*}
        \sum_{a \in \cA: R_{S,a} \neq \emptyset} \sum_{t \in R_{S,a}} \sum_{S' \in \cS} \1[t \in S', a \in \widehat\cA_{t,S'}^\bc] c_j(a,y_t) = \sum_{a \in \cA: R_{S,a} \neq \emptyset} \sum_{S' \in \cS} \sum_{t \in R_{S,a} \cap S': a \in \widehat\cA_{t,S'}^\bc} c_j(a,y_t).
    \end{align*}  

    For notational convenience, we define $R_{S,a,S'} = \{t \in R_{S,a} \cap S' : a \in \widehat\cA_{t,S'}^\bc \}$. For any $t \in R_{S,a,S'} \setminus \{\max R_{S,a,S'}\}$, action $a$ is played and is not eliminated from the candidate action set corresponding to $S'$, which implies that the constraint is satisfied, i.e., $c_j(a,y_t) \le 0$. 
    Therefore, only the final term, $c_j(a,y_{\max R_{S,a,S'}})$, can contribute positively to the sum for each action. The sum is thus bounded by:
    \begin{align*}
        \MoveEqLeft \sum_{a \in \cA: R_{S,a} \neq \emptyset} \sum_{S' \in \cS} \left( \sum_{t \in R_{S,a,S'} \setminus \{\max R_{S,a,S'}\}} c_j(a,y_t) + c_j(a,y_{\max R_{S,a,S'}}) \right) \\
        &\le \sum_{a \in \cA: R_{S,a} \neq \emptyset} \sum_{S' \in \cS} c_j(a,y_{\max R_{S,a,S'}}) \\
        &\le \sum_{a \in \cA: R_{S,a} \neq \emptyset} \sum_{S' \in \cS} 1 \\
        &\le |\cA| |\cS|.
    \end{align*}
\end{proof}

Handling the expectation-based benchmarks $\{\cA_S^{\E[\bc]}\}_{S \in \cS}$, in contrast, requires a more sophisticated approach than simply running parallel copies of Algorithm \ref{alg:elimination-expectation}. The challenge arises because each instantiation of the algorithm has only a local view of an action's performance. An action could appear feasible from the perspective of a subsequence $S$, but within a smaller, overlapping portion of another subsequence $S'$, that same action could be highly problematic. Low (i.e. very negative) constraint violation over the difference $S \setminus S'$ can mask the high constraint violation on the intersection $S \cap S'$, preventing $S$ from eliminating the action. This action would then remain in the union set $U_t$ and continue to harm the performance of subsequence $S'$, even if $S'$ has already eliminated it.

Our solution has the same high-level structure as Algorithm \ref{alg:elimination-realization-subsequence}. We maintain a candidate action set $\widehat\cA_S^{\E[\bc]}$ for each subsequence and gradually eliminate actions from it; the agent selects an action from the union set, $U_t = \cup_{S \in \cS_t} \widehat\cA_S^{\E[\bc]}$, in each round. The key difference lies in a more sophisticated elimination rule.

We will monitor each action's constraint violations more granularly by tracking their performance over the intersections of pairs of subsequence. However, since these intersections overlap, simply summing constraint violations across them would lead to double-counting. To resolve this, we will attribute the action $a_t$ played in each round $t \in [T]$ to a single, uniquely responsible subsequence. This responsible subsequence can be selected in any arbitrary way. To be concrete, we let the collection of subsequences be indexed as $\cS = \{S_1, S_2, \ldots, S_I\}$ where $I = |\cS|$. 
% We credit the subsequence $S_{r_t}$ where the index $r_t$ is defined as:
We credit the subsequence $S_{r_t}$ with the smallest index that is active and for which $a_t$ is a candidate action:
\begin{equation*}
    r_t = \min\left\{ i \in [I]: t \in S_i, a_t \in \widehat\cA_{t,S_i}^{\E[\bc]} \right\}.
\end{equation*}
% This rule deterministically assigns responsibility to the active subsequence with the smallest index whose candidate action set contains the chosen action $a_t$.

% For each action $a \in \cA$, constraint $j \in [J]$, and subsequence $S \in \mathcal{S}$, we let $v^{a,j}_{[1,t] \cap S}$ denote the total violation of the $j$-th constraint when action $a$ is played, accumulated over the subsequence $S$ up to round $t$, i.e.,
% \begin{align*}
%     v^{a,j}_{[1,t] \cap S} \coloneqq \sum_{s=1}^t \1[a_s=a, s \in S] c_j(a,y_s)
% \end{align*}

% We further decompose $v^{a,j}_{[1,t] \cap S}$ based on the responsible subsequences...

For each action $a \in \cA$, constraint $j \in [J]$, and pair of subsequences $(S,S') \in \cS^2$, we define the attributed constraint violation $w^{a,j,S'}_{[1,t] \cap S}$ that tracks the cumulative violation of action $a$ against the $j$-th constraint during the rounds of $S$ up to round $t$ for which subsequence $S'$ is responsible:
\begin{align*}
    w^{a,j,S'}_{[1,t] \cap S} \coloneqq \sum_{s=1}^t \1[a_s=a, s \in S, S_{r_s}=S'] c_j(a,y_s).
\end{align*}

Our new elimination rule is based on monitoring these terms against subsequence-specific thresholds, $\{\tau_{S'}\}_{S' \in \cS}$. An action $a_t$ is eliminated from the candidate action set of its responsible subsequence, $S_{r_t}$, if its attributed constraint violations exceed a subsequence-specific threshold $\tau_{S_{r_t}}$. Specifically, we eliminate $a_t$ from $\widehat\cA_{t,S_{r_t}}^{\E[\bc]}$ if for any constraint $j \in [J]$ and any subsequence $S \in \cS$, the term $w^{a,j,S_{r_t}}_{[1,t] \cap S}$ surpasses $\tau_{S_{r_t}}$. This entire procedure is formally detailed in Algorithm \ref{alg:elimination-expectation-subsequence}.

\begin{algorithm}[H]
    \KwIn{Threshold parameters $\{\tau_S\}_{S \in \cS}$}
    % \KwOut{} 
    % \vspace{.5em}

    Initialize $\widehat\cA^{\E[\bc]}_{1,S} = \cA$ for each $S \in \cS$\;

    For every $a \in \cA$, $j \in [J]$, and $S \in \cS$, set $w^{a,j,S}_\emptyset = 0$\;

    \For{$t=1$ \KwTo $T$}{
        Receive the prediction $p_t$\;

        Aggregate candidate action sets from active subsequences, $U_{t} = \cup_{S \in \cS_t} \widehat\cA_{t,S}^{\E[\bc]}$\;

        Play the $U_t$-constrained best response to $p_t$, $a_t = \CBR^{u}_{U_t}(p_t)$\;

        Determine the responsible subsequence with index $r_t = \min\{i \in [I]: t \in S_i, a_t \in \widehat\cA_{t,S_i}^{\E[\bc]}\}$\;

        Observe the outcome $y_t$, obtain utility $u(a_t, y_t)$ and constraint increments $c_j(a_t, y_t)$, $j \in [J]$\;

        Update $w^{a,j,S'}_{[1,t] \cap S} = w^{a,j,S'}_{[1,t-1] \cap S} + \1[a_t=a, t \in S, S_{r_t}=S'] c_j(a,y_t)$ for every $a \in \cA$, $j \in [J]$, and $S,S' \in \cS$\;

        \If{$\exists j \in [J], S \in \cS, w^{a_t,j,S_{r_t}}_{[1,t] \cap S} > \tau_S$}{
            $\widehat\cA^{\E[\bc]}_{t+1,S_{r_t}} = \widehat\cA^{\E[\bc]}_{t,S_{r_t}} \setminus \{a_t\}$\;
        } 
        \Else{
            $\widehat\cA^{\E[\bc]}_{t+1,S_{r_t}} = \widehat\cA^{\E[\bc]}_{t,S_{r_t}}$\;
        }
        Let $\widehat\cA^{\E[\bc]}_{t+1,S} = \widehat\cA^{\E[\bc]}_{t,S}$ for any non-responsible subsequence $S \in \cS \setminus \{S_{r_t}\}$\;
    }
    
    \caption{Elimination-Based Algorithm for Benchmark Classes $\{\cA^{\E[\bc]}_{S}\}_{S \in \cS}$}
    \label{alg:elimination-expectation-subsequence}
\end{algorithm}

Algorithm \ref{alg:elimination-expectation-subsequence} never eliminates any benchmark action in $\cA^{\E[\bc]}_S$ from the candidate action set $\widehat\cA^{\E[\bc]}_{t,S}$ across the $T$ rounds for any subsequence $S \in \cS$ with high probability over the randomness of the outcomes $y_1,\ldots,y_T$.
\begin{proposition} \label{prop:benchmark-expectation-subsequence}
    Fix an agent in $\cN$ with a utility function $u: \cA \times \cY \to [0,1]$ and $J$ constraint functions $\{c_j: \cA \times \cY \to [-1,1]\}_{j \in [J]}$. Suppose the agent runs Algorithm \ref{alg:elimination-expectation-subsequence} with threshold parameters $\{\tau_S\}_{S \in \cS}$ to compete with the benchmark class $\cA^{\E[\bc]}_S$ over every subsequence $S \in \cS$, then with probability at least $1 - 2|\cA| |\cS| J \sum_{S \in \cS} |S| e^{-\frac{\tau_S^2}{8|S|}}$:
    \begin{align*}
        \cA_{S}^{\E[\bc]} \subseteq \widehat\cA_{t,S}^{\E[\bc]} \subseteq U_t \text{ for every } t \in [T] \text{ and } S \in \cS_t.
    \end{align*}
\end{proposition}
\begin{proof}
    Fix any $j \in [J]$, $(S,S') \in \cS^2$, and $a \in \cA^{\E[\bc]}_{S}$, we have that for every $t \in [T]$:
    \begin{align*}
        \Pr\left[ w^{a,j,S'}_{[1,t] \cap S} > \tau_S \right] &= \Pr\left[ \sum_{s=1}^t \1[a_s=a, s \in S, S_{r_s}=S'] c_j(a,y_s) > \tau_S \right] \\
        &\le \Pr\left[ \sum_{s=1}^t \1[a_s=a, s \in S, S_{r_s}=S'] (c_j(a,y_s)-\E_{y_s \sim Y_s}[c_j(a,y_s)]) > \tau_S \right].
    \end{align*}

    Similarly to the proof of Proposition \ref{prop:benchmark-expectation}, we define the filtration $\cF_t = \sigma\left(\{(x_s,p_s,y_s)\}_{s=1}^t, x_{t+1}, p_{t+1}\right)$ and first show that $\{\1[a_s=a, s \in S, S_{r_s}=S'] (c_j(a,y_s)-\E_{y_s \sim Y_s}[c_j(a,y_s)])\}_{s=1}^T$ is a sequence of martingale differences.
    
    % We will again use the fact that the forecaster produces the prediction $p_t$ in two steps: After round $t-1$, the forecaster produces an implicit prediction function $\tilde p_t: \cX \to \cY$. Then the forecaster observes the feature $x_t$ and predicts $\tilde p_t(x_t)$.

    % As a result, the action played by the agent in round $t$ can be written as $a_t = \CBR^u_{U_t}(\tilde p_t(x_t))$. $U_t$ is the union of some candidate action set maintained by the agent based on the record of the actions up to round $t-1$, so $U_t \in \cF_{t-1}$. $\tilde p_t$ is produced by the forecaster using the information up to round $t-1$, so $\tilde p_t \in \cF_{t-1}$. We define $g_{t-1}(\cdot) = \CBR^u_{U_t}(\tilde p_t(\cdot))$ for convenience. Then $a_t = g_{t-1}(x_t)$ and $g_{t-1} \in \cF_{t-1}$.
    We have that $a_t = \CBR^u_{U_t}(p_t)$. $U_t$ is the union of some candidate action set maintained by the agent based on the record of the actions up to round $t-1$, so $U_t \in \cF_{t-1}$. $p_t \in \cF_{t-1}$ by definition of the filtration. So $a_t \in \cF_{t-1}$ and hence $\1[a_t=a] \in \cF_{t-1}$.

    We have that $\1[t \in S] = \1[h_{S_i}(t,x_t) = 1]$, where $x_t \in \cF_{t-1}$. As a result, $\1[t \in S] \in \cF_{t-1}$.

    % The responsible subsequence at round $t$ has index $r_t = \min\{i \in [I]: t \in S_i, a_t \in \widehat\cA_{t,S_i}^{\E[\bc]}\} = \min\{i \in [I]: h_{S_i}(t,x_t)=1, g_{t-1}(x_t) \in \widehat\cA_{t,S_i}^{\E[\bc]}\}$ where $g_{t-1} \in \cF_{t-1}$ and $\widehat\cA_{t,S_i}^{\E[\bc]} \in \cF_{t-1}$.
    The responsible subsequence at round $t$ has index $r_t = \min\{i \in [I]: t \in S_i, a_t \in \widehat\cA_{t,S_i}^{\E[\bc]}\} = \min\{i \in [I]: h_{S_i}(t,x_t)=1, a_t \in \widehat\cA_{t,S_i}^{\E[\bc]}\}$ where $x_t \in \cF_{t-1}$, $a_t \in \cF_{t-1}$, and $\widehat\cA_{t,S_i}^{\E[\bc]} \in \cF_{t-1}$. Therefore, $\1[S_{r_t} = S'] \in \cF_{t-1}$.

    Consequently, for any $t \in [T]$: 
    % \begin{align*}
    %    \MoveEqLeft \E\left[ \1[a_t=a, t \in S, S_{r_t}=S'] (c_j(a,y_t)-\E_{y_t \sim Y_t}[c_j(a,y_t)]) \mid \cF_{t-1} \right] \\
    %    &= \E\left[ \1[g_{t-1}(x_t)=a, t \in S, S_{\min\{i \in [I]: h_{S_i}(t,x_t)=1, g_{t-1}(x_t) \in \widehat\cA_{t,S_i}^{\E[\bc]}\}}=S'] (c_j(a,y_t)-\E_{y_t \sim Y_t}[c_j(a,y_t)]) \mid \cF_{t-1} \right] \\
    %    &= 0
    % \end{align*}
    \begin{align*}
       \MoveEqLeft \E\left[ \1[a_t=a, t \in S, S_{r_t}=S'] (c_j(a,y_t)-\E_{y_t \sim Y_t}[c_j(a,y_t)]) \mid \cF_{t-1} \right] \\
       &= \1[a_t=a, t \in S, S_{r_t}=S'] \E\left[ c_j(a,y_t)-\E_{y_t \sim Y_t}[c_j(a,y_t)] \mid \cF_{t-1} \right] \\
       &= 0.
    \end{align*}
    This implies that $\{\1[a_s=a, s \in S, S_{r_s}=S'] (c_j(a,y_s)-\E_{y_s \sim Y_s}[c_j(a,y_s)])\}_{s=1}^t$ is a sequence of martingale differences. 
    The subsequence of these terms corresponding to rounds $s \in S$, i.e., $\{\1[a_s=a, s \in S, S_{r_s}=S'] (c_j(a,y_s)-\E_{y_s \sim Y_s}[c_j(a,y_s)])\}_{s \in S: s \le t}$, is also a martingale difference sequence, because the selection rule is predictable with respect to the filtration $\cF_{s-1}$.
    By Azuma-Hoeffding inequality (Lemma \ref{lem:azuma}), we derive that for any $j \in [J]$, $(S,S') \in \cS^2$, $a \in \cA^{\E[\bc]}_{S}$, and $t \in S$:
    \begin{align*}
        \Pr\left[ w^{a,j,S'}_{[1,t] \cap S} > \tau_S \right] &\le \Pr\left[ \sum_{s=1}^t \1[a_s=a, s \in S, S_{r_s}=S'] (c_j(a,y_s)-\E_{y_s \sim Y_s}[c_j(a,y_s)]) > \tau_S \right] \\
        &= \Pr\left[ \sum_{s \in S: s \le t} \1[a_s=a, s \in S, S_{r_s}=S'] (c_j(a,y_s)-\E_{y_s \sim Y_s}[c_j(a,y_s)]) > \tau_S \right] \\
        &\le 2e^{-\frac{\tau^2_S}{8|S|}}.
    \end{align*}

    Using the union bound over every $j \in [J]$, $(S,S') \in \cS^2$, $a \in \cA^{\E[\bc]}_{1:T}$, and $t \in S$, we derive that:
    \begin{align*}
        \Pr\left[ \exists j \in [J], (S,S') \in \cS^2, a \in \cA^{\E[\bc]}_{S}, t \in S, \text{ such that } w^{a,j,S'}_{[1,t] \cap S} > \tau_S \right] &\le \sum_{j \in [J]} \sum_{(S,S') \in \cS^2} \sum_{a \in \cA^{\E[\bc]}_{S}} \sum_{t \in S} 2e^{-\frac{\tau_S^2}{8|S|}} \\
        &\le 2|\cA| |\cS| J \sum_{S \in \cS} |S| e^{-\frac{\tau_S^2}{8|S|}}.
    \end{align*}

    By the elimination rule of Algorithm \ref{alg:elimination-expectation-subsequence}, this implies that the agent never eliminates any action in $\cA^{\E[\bc]}_{S}$ across the $T$ rounds for any subsequence $S \in \cS$ with probability at least $1 - 2|\cA| |\cS| J \sum_{S \in \cS} |S| e^{-\frac{\tau_S^2}{8|S|}}$.
\end{proof}

% Proposition \ref{prop:benchmark-expectation-subsequence} is stated with respect to a single agent. Using the union bound over every agent in $\cN$, we derive that no agents in $\cN$ eliminates any benchmark action in $\cA^{\E[\bc]}_{S}$ across the $T$ rounds for any subsequence $S \in \cS$ with probability at least $1-2|\cA||\cN||\cS|^2JTe^{-\frac{\tau^2}{8T}}$. Therefore, we will set $\tau = 4\sqrt{T\ln\frac{|\cA||\cN||\cS|^2JT}{\delta}}$ where $\delta \in (0,1)$ is a (presumably small) constant. This leads to the following corollary regarding every agent, which holds with probability at least $1-\delta$.
Proposition \ref{prop:benchmark-expectation-subsequence} provides a guarantee for a single agent. To extend this to all agents in $\cN$, we apply a union bound. By setting the thresholds $\tau_S = 4\sqrt{|S|\ln\frac{|\cA||\cN||\cS|^2J|S|}{\delta}}$ for a desired failure probability $\delta \in (0,1)$, we arrive at the following corollary, which holds with probability at least $1-\delta$.
\begin{corollary} \label{cor:benchmark-expectation-subsequence}
    Let $\cS$ be a collection of subsequences. Let $\cN$ be a set of agents, where each agent is equipped with a utility function $u: \cA \times \cY \to [0,1]$ and $J$ constraint functions $\{c_j: \cA \times \cY \to [-1,1]\}_{j \in [J]}$. Each agent runs Algorithm \ref{alg:elimination-expectation-subsequence} to compete with the benchmark class $\cA^{\E[\bc]}_S$ over every subsequence $S \in \cS$, where the threshold parameter for each subsequence $S$ is set to $4\sqrt{|S|\ln\frac{|\mathcal{A}||\mathcal{N}||\mathcal{S}|^2J|S|}{\delta}}$. Then with probability at least $1-\delta$, for every agent:
    \begin{align*}
        \cA_{S}^{\E[\bc]} \subseteq \widehat\cA_{t,S}^{\E[\bc]} \subseteq U_t \text{ for every } t \in [T] \text{ and } S \in \cS_t.
    \end{align*}
\end{corollary}

We now bound the cumulative constraint violation of Algorithm \ref{alg:elimination-expectation-subsequence}. The reasoning is a granular extension of the argument for Theorem \ref{thm:CCV-expectation}. In Algorithm \ref{alg:elimination-expectation-subsequence}, an action $a$ is eliminated from the candidate action set of the responsible subsequence $S'$ only when one of its attributed constraint violations, $w_{[1,t] \cap S}^{a,j,S'}$, surpasses the threshold $\tau_{S'}$. This implies that each attributed constraint violation term, $w_{[1,t] \cap S}^{a,j,S'}$, is bounded by at most $\tau_{S'}+1$. Summing this bound over all possible responsible subsequences $S' \in \cS$ and all actions $a \in \cA$ gives the final bound on $\CCV(S)$.
\begin{theorem} \label{thm:CCV-expectation-subsequence}
    Let $\cS$ be a collection of subsequences. Let $\cN$ be a set of agents, where each agent is equipped with a utility function $u: \cA \times \cY \to [0,1]$ and $J$ constraint functions $\{c_j: \cA \times \cY \to [-1,1]\}_{j \in [J]}$. Each agent runs Algorithm \ref{alg:elimination-expectation-subsequence} to compete with the benchmark class $\cA^{\E[\bc]}_S$ over every subsequence $S \in \cS$, where the threshold parameter for each subsequence $S$ is set to $4\sqrt{|S|\ln\frac{|\mathcal{A}||\mathcal{N}||\mathcal{S}|^2J|S|}{\delta}}$. Then with probability at least $1-\delta$, the cumulative constraint violation of any agent over any subsequence $S \in \cS$ is bounded by:
    \begin{align*}
        \CCV(S) \le \sum_{S \in \cS} 4 |\cA| \sqrt{|S|\ln\frac{|\mathcal{A}||\mathcal{N}||\mathcal{S}|^2J|S|}{\delta}} + |\cA| |\cS|. 
        % \le 4 |\cA||\cS| \sqrt{T\ln\frac{|\cA||\cN||\cS|^2JT}{\delta}} + |\cA| |\cS|
    \end{align*}
\end{theorem}
\begin{proof}
    By Corollary \ref{cor:benchmark-expectation-subsequence}, no agents in $\cN$ eliminates any benchmark action in $\cA_{S}^{\E[\bc]}$ across the $T$ rounds for any subsequence $S \in \cS$ with probability at least $1-\delta$. We will condition on this event, which ensures that the Algorithm \ref{alg:elimination-expectation-subsequence} is well-behaved and avoids the risk of eliminating all actions in $\cA$.    

    Fix any $j \in [J]$ and $S \in \cS$. We will bound the cumulative violation against the $j$-th constraint over $S$, i.e., $\sum_{t \in S} c_j(a_t, y_t)$. The final result follows by taking the maximum over all $j \in [J]$.
    
    We will partition the sum based on the action played and the responsible subsequence. For each action $a \in \cA$ and subsequence $S' \in \cS$, we define $R_{S,a,S'} = \{t \in [T] : a_t=a, t \in S, S_{r_t}=S'\}$. Then we have that:
    \begin{align*}
        \sum_{t \in S} c_j(a_t, y_t) &= \sum_{t=1}^T \1[t \in S] c_j(a_t, y_t) \\
        &= \sum_{a \in \cA} \sum_{S' \in \cS} \sum_{t=1}^T \1[a_t=a, t \in S, S_{r_t}=S'] c_j(a, y_t) \\
        &= \sum_{a \in \cA} \sum_{S' \in \cS} \sum_{t=1}^{\max R_{S,a,S'}} \1[a_t=a, t \in S, S_{r_t}=S'] c_j(a, y_t) \\
        &= \sum_{a \in \cA} \sum_{S' \in \cS} \left( \sum_{t=1}^{\max R_{S,a,S'}-1} \1[a_t=a, t \in S, S_{r_t}=S'] c_j(a, y_t) + c_j(a, y_{\max R_{S,a,S'}}) \right) \\
        &= \sum_{a \in \cA} \sum_{S' \in \cS} (w_{[1,\max R_{S,a,S'}-1] \cap S}^{a,j,S'} + c_j(a, y_{\max R_{S,a,S'}})) \\
        &\le \sum_{a \in \cA} \sum_{S' \in \cS} (\tau_{S'} + 1) \\
        &= \sum_{S \in \cS} 4 |\cA| \sqrt{|S|\ln\frac{|\mathcal{A}||\mathcal{N}||\mathcal{S}|^2J|S|}{\delta}} + |\cA||\cS|.
    \end{align*}
\end{proof}

Next, we focus on the regret against the benchmark class $\cA_S^\bc$ over every subsequence $S \in \cS$. Similarly to Section \ref{sec:regret}, we will obtain the same regret bounds for the benchmark class $\cA_S^{\E[\bc]}$ and relegate the results to Appendix \ref{app:subsequence}.

To obtain the regret bounds, we will need the predictions to be decision calibrated, conditional on every subsequence in $\cS$.

\begin{definition}[$(\cN,\cS,\alpha_2)$-Decision Calibration] \label{def:decision-calibration-realization-subsequence}
    Let $\cS$ be a collection of subsequences. Let $\cN$ be a set of agents, where each agent is equipped with a utility function $u: \cA \times \cY \to [0,1]$ and $J$ constraint functions $\{c_j: \cA \times \cY \to [-1,1]\}_{j \in [J]}$. Each agent will run Algorithm \ref{alg:elimination-realization-subsequence} to compete with the benchmark class $\cA^\bc_S$ over every subsequence $S \in \cS$. Each agent will maintain a candidate action set $\{U_t\}_{t \in [T]}$. Let $\alpha_2: \bR \to \bR$. We say that a sequence of predictions $p_1,\ldots,p_T$ is $(\cN,\cS,\alpha_2)$-decision calibrated with respect to a sequence of outcomes $y_1,\ldots,y_T$ if for every $S \in \cS$, $a\in \cA$, and $(u,\bc) \in \cN$:
    \[
        \left\| \sum_{t=1}^T \1[t \in S, \CBR^u_{U_t}(p_t)=a] (p_t - y_t) \right\|_\infty \leq \alpha_2(T^{u,\bc,S}(a)).
    \]
    where $T^{u,\bc,S}(a) = \sum_{t=1}^T \1[t \in S, \CBR^u_{U_t}(p_t)=a]$.
\end{definition}

Similarly to Theorems \ref{thm:swap-regret-realization} in the single-sequence case, decision calibration leads to the constrained swap regret bound.
\begin{theorem} \label{thm:swap-regret-realization-subsequence}
    Let $\cS$ be a collection of subsequences. Let $\cN$ be a set of agents, where each agent is equipped with a utility function $u: \cA \times \cY \to [0,1]$ and $J$ constraint functions $\{c_j: \cA \times \cY \to [-1,1]\}_{j \in [J]}$. The utility functions are linear and $L$-Lipschitz in the second argument. Each agent runs Algorithm \ref{alg:elimination-realization-subsequence} to compete with the benchmark class $\cA^\bc_S$ over every subsequence $S \in \cS$. If the sequence of predictions $p_1,\ldots,p_T$ is $(\cN,\cS,\alpha_2)$-decision calibrated, then the $\cA^\bc_{S}$-constrained swap regret of any agent over any subsequence $S \in \cS$ is bounded by:
    \begin{align*}
        \Reg_\swap(u,\cA_{S}^\bc,S) \le 2 L |\cA| \alpha_2(|S|/|\cA|).
    \end{align*}
\end{theorem}

We will again instantiate \textsc{Unbiased-Prediction} to make decision calibrated predictions.

\begin{theorem} \label{thm:unbiased-algorithm-realization-subsequence}
    Let $\cS$ be a collection of subsequences. Let $\cN$ be a set of agents, where each agent is equipped with a utility function $u: \cA \times \cY \to [0,1]$ and $J$ constraint functions $\{c_j: \cA \times \cY \to [-1,1]\}_{j \in [J]}$. Each agent will run Algorithm \ref{alg:elimination-realization-subsequence} to compete with the benchmark class $\cA^\bc_S$ over every subsequence $S \in \cS$. Each agent will maintain a candidate action set $\{U_t\}_{t \in [T]}$. 
    There is an instantiation of \textsc{Unbiased-Prediction} \citep{noarov2023highdimensional}
    ---which we call \textsc{Decision-Calibration-Subsequence}---
    producing predictions $p_1,...,p_T \in \cY$ satisfying, for any sequence of outcomes $y_1,...,y_T \in \cY$, the following holds  
    with probability at least $1-\delta$ for any $(u,\bc) \in \cN$, $a \in \cA$, and $S \in \cS$:
    \[
        \left\| \sum_{t=1}^T \1[t \in S, \CBR^u_{U_{t}}(p_t) = a] (p_t - y_t) \right\|_\infty \leq O\left( \ln(d|\cA||\cN||\cS|T) + \sqrt{\ln(d|\cA||\cN||\cS|T) \cdot T^{u,\bc,S}(a)} + \sqrt{\ln\frac{d|\cA||\cN||\cS|}{\delta} \cdot |S|}\right).
    \]
    
\end{theorem}

Substituting the above bounds into Theorems \ref{thm:swap-regret-realization-subsequence}, we arrive at the following corollary bounding the $\cA^\bc_{S}$-constrained swap regret over every subsequence $S \in \cS$.

\begin{corollary} \label{cor:swap-regret-realization-subsequence}
    Let $\cN$ be a set of agents, where each agent is equipped with a utility function $u: \cA \times \cY \to [0,1]$ and $J$ constraint functions $\{c_j: \cA \times \cY \to [-1,1]\}_{j \in [J]}$. The utility functions are linear and $L$-Lipschitz in the second argument. Each agent will run Algorithm \ref{alg:elimination-realization-subsequence} to compete with the benchmark class $\cA^\bc_S$ over every subsequence $S \in \cS$.
    The sequence of predictions $p_1,\ldots,p_T$ outputted by \textsc{Decision-Calibration-Subsequence} ensures that with probability at least $1-\delta$, the $\cA^\bc_{S}$-constrained swap regret of any agent over any subsequence $S \in \cS$ is bounded by:
    \begin{align*}
        \Reg_\swap(u,\cA_{S}^\bc,S) \le O\left( L |\cA| \sqrt{|S| \ln\frac{d|\cA||\cN||\cS|T}{\delta}} \right).
    \end{align*}
\end{corollary}

Since constrained swap regret is stronger than constrained external regret, we can also bound the $\cA^\bc_{S}$-constrained external regret over every subsequence $S \in \cS$.
\begin{corollary} \label{cor:external-regret-realization-subsequence}
    Let $\cN$ be a set of agents, where each agent is equipped with a utility function $u: \cA \times \cY \to [0,1]$ and $J$ constraint functions $\{c_j: \cA \times \cY \to [-1,1]\}_{j \in [J]}$. The utility functions are linear and $L$-Lipschitz in the second argument. Each agent will run Algorithm \ref{alg:elimination-realization-subsequence} to compete with the benchmark class $\cA^\bc_S$ over every subsequence $S \in \cS$.
    The sequence of predictions $p_1,\ldots,p_T$ outputted by \textsc{Decision-Calibration-Subsequence} ensures that with probability at least $1-\delta$, the $\cA^\bc_{S}$-constrained external regret of any agent over any subsequence $S \in \cS$ is bounded by:
    \begin{align*}
        \Reg_\ext(u,\cA_{S}^\bc,S) \le O\left( L |\cA| \sqrt{|S| \ln\frac{d|\cA||\cN||\cS|T}{\delta}} \right).
    \end{align*}
\end{corollary}

\bibliographystyle{ACM-Reference-Format}
\bibliography{references}

\newpage
\appendix

\section{Omitted Results in Section \ref{sec:regret}} \label{app:regret}

In this section, we show how to produce predictions that lead to $\tilde O(|\cA|\sqrt{T})$ regret against the benchmark class $\cA^{\E[\bc]}_{1:T}$.
By Corollary \ref{cor:benchmark-expectation}, with probability at least $1-\delta$ over the randomness of the outcomes $y_1,\ldots,y_T$, $\cA_{1:T}^{\E[\bc]} \subseteq \widehat\cA_t^{\E[\bc]}$ for every agent and every $t \in [T]$. We will condition on this event throughout this section. Therefore, all the guarantees we give in this section hold with probability at least $1-\delta$.

Similarly to Definition \ref{def:decision-calibration-realization}, we will need the following notion of conditional unbiasedness.

\begin{definition}[$(\cN,\beta_1)$-Decision Calibration] \label{def:decision-calibration-expectation}
    Let $\cN$ be a set of agents, where each agent is equipped with a utility function $u: \cA \times \cY \to [0,1]$ and $J$ constraint functions $\{c_j: \cA \times \cY \to [-1,1]\}_{j \in [J]}$. Each agent runs Algorithm \ref{alg:elimination-expectation} with threshold parameter $4\sqrt{T\ln\frac{|\cA||\cN|JT}{\delta}}$ to maintain a candidate action set $\widehat\cA_t^{\E[\bc]}$ and $\widehat\cA_t^{\E[\bc]}$-constrained best respond to the prediction $p_t$ in every round $t \in [T]$. Let $\beta_1: \bR \to \bR$. We say that a sequence of predictions $p_1,\ldots,p_T$ is $(\cN,\beta_1)$-decision calibrated with respect to a sequence of outcomes $y_1,\ldots,y_T$ if for every $a\in \cA$ and $(u,\bc) \in \cN$:
    \[
        \left\| \sum_{t=1}^T \1[\CBR^u_{\widehat\cA_{t}^{\E[\bc]}}(p_t)=a] (p_t - y_t) \right\|_\infty \leq \beta_1(T^{u,\E[\bc]}(a))
    \]
    where $T^{u,\E[\bc]}(a) = \sum_{t=1}^T \1[\CBR^u_{\widehat\cA_t^{\E[\bc]}}(p_t)=a]$.
\end{definition}

Both Definitions \ref{def:decision-calibration-realization} and \ref{def:decision-calibration-expectation} require the predictions to be unbiased conditional on the actions chosen by agents. They differ only in agents' decision rules. In Definition \ref{def:decision-calibration-realization}, agents $\widehat\cA_t^\bc$-constrained best respond to the predictions to compete with the benchmark class $\cA_{1:T}^\bc$, whereas in Definition \ref{def:decision-calibration-expectation}, agents $\widehat\cA_t^{\E[\bc]}$-constrained best respond to the predictions to compete with the benchmark class $\cA_{1:T}^{\E[\bc]}$.
For simplicity, we refer to both conditions as decision calibration. 

Similarly to Theorem \ref{thm:swap-regret-realization}, decision calibration implies no $\cA^{\E[\bc]}_{1:T}$-constrained swap regret.
\begin{theorem} \label{thm:swap-regret-expectation}
    Let $\cN$ be a set of agents, where each agent is equipped with a utility function $u: \cA \times \cY \to [0,1]$ and $J$ constraint functions $\{c_j: \cA \times \cY \to [-1,1]\}_{j \in [J]}$. The utility functions are linear and $L$-Lipschitz in the second argument. Each agent runs Algorithm \ref{alg:elimination-expectation} with threshold parameter $4\sqrt{T\ln\frac{|\cA||\cN|JT}{\delta}}$ to compete with the benchmark class $\cA^{\E[\bc]}_{1:T}$.
    If the sequence of predictions $p_1,\ldots,p_T$ is $(\cN,\beta_1)$-decision calibrated, then with probability at least $1-\delta$, the $\cA^{\E[\bc]}_{1:T}$-constrained swap regret of every agent is bounded by:
    \begin{align*}
        \Reg_\swap(u,\cA_{1:T}^{\E[\bc]},1:T) \le 2 L |\cA| \beta_1(T/|\cA|)
    \end{align*}
\end{theorem}

The proof is similar to the proof of Theorem \ref{thm:swap-regret-realization}. For any action modification rule $\phi: \cA \to \cA_{1:T}^{\E[\bc]}$, we can write:
\begin{align*}
    \MoveEqLeft \sum_{t=1}^T \left( u(\phi(a_t),y_t) - u(a_t,y_t) \right) \\
    &= \sum_{t=1}^T \left( u(\phi(a_t),y_t) - u(\phi(a_t),p_t) \right) + \sum_{t=1}^T \left( u(\phi(a_t),p_t) - u(a_t,p_t) \right) + \sum_{t=1}^T \left( u(a_t,p_t) - u(a_t,y_t) \right)
\end{align*}

Conditioning on the event in Corollary \ref{cor:benchmark-expectation} and using the fact that $\phi(a_t) \in \cA_{1:T}^{\E[\bc]}$, we have that $\phi(a_t) \in \widehat\cA_t^{\E[\bc]}$ for every $t \in [T]$. Since agents $\widehat\cA_t^{\E[\bc]}$-constrained best response to the prediction $p_t$ in round $t$, $a_t = \argmax_{b \in \widehat\cA_t^{\E[\bc]}} u(b,p_t)$. This implies that:
\begin{align*}
    u(\phi(a_t),p_t) \le u(a_t,p_t).
\end{align*}
So, the second part in the decomposition is bounded by 0.

Thus, it suffices to bound the other two parts using decision calibration.

\begin{lemma} \label{lem:decision-calibration-expectation}
    If the sequence of predictions $p_1,\ldots,p_T$ is $(\cN,\beta_1)$-decision calibrated, then for any $(u,\bc) \in \cN$:
    \begin{align*}
        \left| \sum_{t=1}^T (u(a_t,p_t) - u(a_t,y_t)) \right| \le L |\cA| \beta_1(T/|\cA|)
    \end{align*}
\end{lemma}
\begin{proof}
    The proof is similar to the proof of Lemma \ref{lem:decision-calibration-realization}. Using the linearity of $u$, we can write:
    \begin{align*}
        \left| \sum_{t=1}^T (u(a_t,p_t) - u(a_t,y_t)) \right| &= \left| \sum_{a \in \cA}\sum_{t=1}^T \1[a_t=a] (u(a,p_t) - u(a,y_t)) \right| \\
        &= \left| \sum_{a \in \cA} \left( u\left(a, \sum_{t=1}^T \1[a_t=a] p_t \right) - u\left(a, \sum_{t=1}^T \1[a_t=a] y_t\right) \right)\right| \\
        &\le \sum_{a \in \cA} \left| u\left(a, \sum_{t=1}^T \1[a_t=a] p_t \right) - u\left(a, \sum_{t=1}^T \1[a_t=a] y_t\right) \right| \\
        &\le \sum_{a \in \cA} L \left\| \sum_{t=1}^T \1[a_t=a] (p_t-y_t) \right\|_\infty \\
        &\le L \sum_{a \in \cA} \beta_1(T^{u,\E[\bc]}(a))
    \end{align*}
    where the first inequality follows from the triangle inequality, the second inequality follows from $L$-Lipschitzness of $u$, and the third inequality follows from $(\cN,\beta_1)$-decision calibration. By concavity of $\beta_1$ and the fact that $\sum_{a \in \cA} T^{u,\E[\bc]}(a) = \sum_{a \in \cA} \sum_{t=1}^T \1[a_t=a] = T$, this expression is at most:
    \begin{align*}
        L |\cA| \beta_1(T/|\cA|)
    \end{align*}
\end{proof}

\begin{lemma} \label{lem:decision-calibration-expectation-swap}
    If the sequence of predictions $p_1,\ldots,p_T$ is $(\cN,\beta_1)$-decision calibrated, then for any $(u,\bc) \in \cN$ and $\phi: \cA \to \cA_{1:T}^{\E[\bc]}$:
    \begin{align*}
        \left| \sum_{t=1}^T (u(\phi(a_t),p_t) - u(\phi(a_t),y_t)) \right| \le L |\cA| \beta_1(T/|\cA|).
    \end{align*}
\end{lemma}
\begin{proof}
    Using the linearity of $u$, we can write:
    \begin{align*}
        \left| \sum_{t=1}^T (u(\phi(a_t),p_t) - u(\phi(a_t),y_t)) \right| &= \left| \sum_{a \in \cA}\sum_{t=1}^T \1[a_t=a] (u(\phi(a),p_t) - u(\phi(a),y_t)) \right| \\
        &= \left| \sum_{a \in \cA} \left( u\left(\phi(a), \sum_{t=1}^T \1[a_t=a] p_t \right) - u\left(\phi(a), \sum_{t=1}^T \1[a_t=a] y_t\right) \right)\right| \\
        &\le \sum_{a \in \cA} \left| u\left(\phi(a), \sum_{t=1}^T \1[a_t=a] p_t \right) - u\left(\phi(a), \sum_{t=1}^T \1[a_t=a] y_t\right) \right| \\
        &\le \sum_{a \in \cA} L \left\| \sum_{t=1}^T \1[a_t=a] (p_t-y_t) \right\|_\infty \\
        &\le L \sum_{a \in \cA} \beta_1(T^{u,\E[\bc]}(a)) \\
        &\le L |\cA| \beta_1(T/|\cA|)
    \end{align*}
    where the first inequality follows from the triangle inequality, the second inequality follows from $L$-Lipschitzness of $u$, the third inequality follows from $(\cN,\beta_1)$-decision calibration, and the fourth inequality follows from concavity of $\beta_1$.
\end{proof}

We can now complete the proof of Theorem \ref{thm:swap-regret-expectation}.
\begin{proof}
    Fix any $(u,\bc) \in \cN$ and $\phi: \cA \to \cA^{\E[\bc]}_{1:T}$. Applying Lemmas \ref{lem:decision-calibration-expectation} and \ref{lem:decision-calibration-expectation-swap} to the decomposition of the regret against $\phi$, we have that:
    \begin{align*}
        \MoveEqLeft \sum_{t=1}^T \left( u(\phi(a_t),y_t) - u(a_t,y_t) \right) \\
        &= \sum_{t=1}^T \left( u(\phi(a_t),y_t) - u(\phi(a_t),p_t) \right) + \sum_{t=1}^T \left( u(\phi(a_t),p_t) - u(a_t,p_t) \right) + \sum_{t=1}^T \left( u(a_t,p_t) - u(a_t,y_t) \right) \\
        &\le 2 L |\cA| \beta_1(T/|\cA|)
    \end{align*}
\end{proof}

We will instantiate \textsc{Unbiased-Prediction} to make decision calibrated predictions as defined in Definition \ref{def:decision-calibration-expectation}. We will refer to this instantiation as \textsc{Decision-Calibration-Expectation}. It has the following guarantees.

\begin{theorem} \label{thm:unbiased-algorithm-expectation}
    Let $\cN$ be a set of agents, where each agent is equipped with a utility function $u: \cA \times \cY \to [0,1]$ and $J$ constraint functions $\{c_j: \cA \times \cY \to [-1,1]\}_{j \in [J]}$. Each agent will run Algorithm \ref{alg:elimination-expectation} with threshold parameter $4\sqrt{T\ln\frac{|\cA||\cN|JT}{\delta}}$ to compete with the benchmark class $\cA^{\E[\bc]}_{1:T}$.
    There is an instantiation of \textsc{Unbiased-Prediction} \citep{noarov2023highdimensional}
    ---which we call \textsc{Decision-Calibration-Expectation}---
    producing predictions $p_1,...,p_T \in \Delta \cY$ satisfying, with probability at least $1-\delta$, for any $(u,\bc) \in \cN$ and $a \in \cA$:
    \[
        \left\| \sum_{t=1}^T \1[\CBR^u_{\widehat\cA_t^{\E[\bc]}}(p_t) = a] (p_t - y_t) \right\|_\infty \leq O\left( \ln(d|\cA||\cN|T) + \sqrt{\ln(d|\cA||\cN|T) \cdot T^{u,\E[\bc]}(a)} + \sqrt{\ln\frac{d|\cA||\cN|}{\delta} \cdot T}\right)
    \]
\end{theorem}

Substituting the above bounds into Theorem \ref{thm:swap-regret-expectation}, we arrive at the following corollary bounding the $\cA^{\E[\bc]}_{1:T}$-constrained swap regret.

\begin{corollary}\label{cor:swap-regret-expectation}
    Let $\cN$ be a set of agents, where each agent is equipped with a utility function $u: \cA \times \cY \to [0,1]$ and $J$ constraint functions $\{c_j: \cA \times \cY \to [-1,1]\}_{j \in [J]}$. The utility functions are linear and $L$-Lipschitz in the second argument. Each agent runs Algorithm \ref{alg:elimination-expectation} with threshold parameter $4\sqrt{T\ln\frac{|\cA||\cN|JT}{\delta}}$ to compete with the benchmark class $\cA^{\E[\bc]}_{1:T}$.
    The sequence of predictions $p_1,\ldots,p_T$ outputted by \textsc{Decision-Calibration-Expectation}  ensures that with probability at least $1-\delta$, the $\cA^{\E[\bc]}_{1:T}$-constrained swap regret of every agent is bounded by:
    \begin{align*}
        \Reg_\swap(u,\cA_{1:T}^{\E[\bc]},1:T) \le O\left( L |\cA| \sqrt{T \ln\frac{d|\cA||\cN|T}{\delta}} \right)
    \end{align*}
\end{corollary}

Since constrained swap regret is stronger than constrained external regret, we can also bound the $\cA^{\E[\bc]}_{1:T}$-constrained external regret.
\begin{corollary}\label{cor:external-regret-expectation}
    Let $\cN$ be a set of agents, where each agent is equipped with a utility function $u: \cA \times \cY \to [0,1]$ and $J$ constraint functions $\{c_j: \cA \times \cY \to [-1,1]\}_{j \in [J]}$. The utility functions are linear and $L$-Lipschitz in the second argument. Each agent runs Algorithm \ref{alg:elimination-expectation} with threshold parameter $4\sqrt{T\ln\frac{|\cA||\cN|JT}{\delta}}$ to compete with the benchmark class $\cA^{\E[\bc]}_{1:T}$.
    The sequence of predictions $p_1,\ldots,p_T$ outputted by \textsc{Decision-Calibration-Expectation}  ensures that with probability at least $1-\delta$, the $\cA^{\E[\bc]}_{1:T}$-constrained external regret of every agent is bounded by:
    \begin{align*}
        \Reg_\ext(u,\cA_{1:T}^{\E[\bc]},1:T) \le O\left( L |\cA| \sqrt{T \ln\frac{d|\cA||\cN|T}{\delta}} \right)
    \end{align*}
\end{corollary}

\section{Proof of Theorem \ref{thm:swap-regret-realization-subsequence}}
Similarly to the proof of Theorem \ref{thm:swap-regret-realization}, for any subsequence $S \in \cS$ and action modification rule $\phi: \cA \to \cA_{S}^\bc$, we can decompose the regret against $\phi$ into three parts as:
\begin{align*}
    \MoveEqLeft \sum_{t \in S} \left( u(\phi(a_t),y_t) - u(a_t,y_t) \right) \\
    &= \sum_{t \in S} \left( u(\phi(a_t),y_t) - u(\phi(a_t),p_t) \right) + \sum_{t \in S} \left( u(\phi(a_t),p_t) - u(a_t,p_t) \right) + \sum_{t \in S} \left( u(a_t,p_t) - u(a_t,y_t) \right)
\end{align*}

By Proposition \ref{prop:benchmark-realization-subsequence} and the fact that $\phi(a_t) \in \cA_S^\bc$,  $\phi(a_t) \in U_t$ for every $t \in S$. Since agents $U_t$-constrained best response to the prediction $p_t$ in round $t$, $a_t = \argmax_{b \in U_t} u(b,p_t)$. This implies that:
\begin{align*}
    u(\phi(a_t),p_t) \le u(a_t,p_t).
\end{align*}
So, the second part in the decomposition is upper bounded by 0.

Thus, it suffices to bound the other two parts using decision calibration.

\begin{lemma} \label{lem:decision-calibration-realization-subsequence}
    If the sequence of predictions $p_1,\ldots,p_T$ is $(\cN,\cS,\alpha_2)$-decision calibrated, then for any $(u,\bc) \in \cN$ and $S \in \cS$:
    \begin{align*}
        \left| \sum_{t \in S} (u(a_t,p_t) - u(a_t,y_t)) \right| \le L |\cA| \alpha_2(|S|/|\cA|).
    \end{align*}
\end{lemma}
\begin{proof}
    Using the linearity of $u$, we can write:
    \begin{align*}
        \left| \sum_{t \in S} (u(a_t,p_t) - u(a_t,y_t)) \right| &= \left| \sum_{a \in \cA}\sum_{t=1}^T \1[t \in S,a_t=a] (u(a,p_t) - u(a,y_t)) \right| \\
        &= \left| \sum_{a \in \cA} \left( u\left(a, \sum_{t=1}^T \1[t \in S,a_t=a] p_t \right) - u\left(a, \sum_{t=1}^T \1[t \in S,a_t=a] y_t\right) \right)\right| \\
        &\le \sum_{a \in \cA} \left| u\left(a, \sum_{t=1}^T \1[t \in S,a_t=a] p_t \right) - u\left(a, \sum_{t=1}^T \1[t \in S,a_t=a] y_t\right) \right| \\
        &\le \sum_{a \in \cA} L \left\| \sum_{t=1}^T \1[t \in S,a_t=a] (p_t-y_t) \right\|_\infty \\
        &\le L \sum_{a \in \cA} \alpha_2(T^{u,\bc,S}(a)).
    \end{align*}
    where the first inequality follows from the triangle inequality, the second inequality follows from $L$-Lipschitzness of $u$, and the third inequality follows from $(\cN,\cS,\alpha_2)$-decision calibration. By concavity of $\alpha_2$ and the fact that $\sum_{a \in \cA} T^{u,\bc,S}(a) = \sum_{a \in \cA} \sum_{t=1}^T \1[t \in S, a_t=a] = |S|$, this expression is at most:
    \begin{align*}
        L |\cA| \alpha_2(|S|/|\cA|).
    \end{align*}
\end{proof}

\begin{lemma} \label{lem:decision-calibration-realization-subsequence-swap}
    If the sequence of predictions $p_1,\ldots,p_T$ is $(\cN,\cS,\alpha_2)$-decision calibrated, then for any $(u,\bc) \in \cN$ and $S \in \cS$:
    \begin{align*}
        \left| \sum_{t \in S} (u(\phi(a_t),p_t) - u(\phi(a_t),y_t)) \right| \le L |\cA| \alpha_2(|S|/|\cA|).
    \end{align*}
\end{lemma}
\begin{proof}
    Using the linearity of $u$, we can write:
    \begin{align*}
        \left| \sum_{t \in S} (u(\phi(a_t),p_t) - u(\phi(a_t),y_t)) \right| &= \left| \sum_{a \in \cA}\sum_{t=1}^T \1[t \in S,a_t=a] (u(\phi(a),p_t) - u(\phi(a),y_t)) \right| \\
        &= \left| \sum_{a \in \cA} \left( u\left(\phi(a), \sum_{t=1}^T \1[t \in S,a_t=a] p_t \right) - u\left(\phi(a), \sum_{t=1}^T \1[t \in S,a_t=a] y_t\right) \right)\right| \\
        &\le \sum_{a \in \cA} \left| u\left(\phi(a), \sum_{t=1}^T \1[t \in S,a_t=a] p_t \right) - u\left(\phi(a), \sum_{t=1}^T \1[t \in S,a_t=a] y_t\right) \right| \\
        &\le \sum_{a \in \cA} L \left\| \sum_{t=1}^T \1[t \in S,a_t=a] (p_t-y_t) \right\|_\infty \\
        &\le L \sum_{a \in \cA} \alpha_2(T^{u,\bc,S}(a)) \\
        &\le L |\cA| \alpha_2(|S|/|\cA|)
    \end{align*}
    where the first inequality follows from the triangle inequality, the second inequality follows from $L$-Lipschitzness of $u$, the third inequality follows from $(\cN,\cS,\alpha_2)$-decision calibration, and the fourth inequality follows from concavity of $\alpha_2$.
\end{proof}

We can now complete the proof of Theorem \ref{thm:swap-regret-realization-subsequence}.
\begin{proof}
    Fix any $(u,\bc) \in \cN$, $S \in \cS$, and $\phi: \cA \to \cA^\bc_{S}$. Applying Lemmas \ref{lem:decision-calibration-realization-subsequence} and \ref{lem:decision-calibration-realization-subsequence-swap} to the decomposition of the regret against $\phi$, we have that:
    \begin{align*}
        \sum_{t \in S} \left( u(a,y_t) - u(a_t,y_t) \right) &= \sum_{t \in S} \left( u(a,y_t) - u(a,p_t) \right) + \sum_{t \in S} \left( u(a,p_t) - u(a_t,p_t) \right) + \sum_{t \in S} \left( u(a_t,p_t) - u(a_t,y_t) \right) \\
        &\le 2 L |\cA| \alpha_2(|S|/|\cA|).
\end{align*}
\end{proof}

\section{Omitted Results in Section \ref{sec:subsequence}} \label{app:subsequence}

In this section, we derive regret bounds for the benchmark classes $\{\cA^{\E[\bc]}_{S}\}_{S \in \cS}$.
By Corollary \ref{cor:benchmark-expectation-subsequence}, with probability at least $1-\delta$ over the randomness of the outcomes $y_1,\ldots,y_T$, $\cA_{S}^{\E[\bc]} \subseteq \widehat\cA_{t,S}^{\E[\bc]} \subseteq U_t$ for every agent, every $t \in [T]$, and every $S \in \cS_t$. We will condition on this event throughout this section. Therefore, all the guarantees we give in this section hold with probability at least $1-\delta$.

Similarly to Definition \ref{def:decision-calibration-realization-subsequence}, we will need the following notion of conditional unbiasedness.

\begin{definition}[$(\cN,\cS,\beta_2)$-Decision Calibration] \label{def:decision-calibration-expectation-subsequence}
    Let $\cS$ be a collection of subsequences. Let $\cN$ be a set of agents, where each agent is equipped with a utility function $u: \cA \times \cY \to [0,1]$ and $J$ constraint functions $\{c_j: \cA \times \cY \to [-1,1]\}_{j \in [J]}$. Each agent will run Algorithm \ref{alg:elimination-expectation-subsequence} to compete with the benchmark class $\cA^{\E[\bc]}_S$ over every subsequence $S \in \cS$, where the threshold parameter for each subsequence $S$ is set to $4\sqrt{|S|\ln\frac{|\mathcal{A}||\mathcal{N}||\mathcal{S}|^2J|S|}{\delta}}$. Each agent will maintain a candidate action set $\{U_t\}_{t \in [T]}$. Let $\beta_2: \bR \to \bR$. We say that a sequence of predictions $p_1,\ldots,p_T$ is $(\cN,\cS,\beta_2)$-decision calibrated with respect to a sequence of outcomes $y_1,\ldots,y_T$ if for every $S \in \cS$, $a\in \cA$, and $(u,\bc) \in \cN$:
    \[
        \left\| \sum_{t=1}^T \1[t \in S, \CBR^u_{U_t}(p_t)=a] (p_t - y_t) \right\|_\infty \leq \beta_2(T^{u,\E[\bc],S}(a)),
    \]
    where $T^{u,\E[\bc],S}(a) = \sum_{t=1}^T \1[t \in S, \CBR^u_{U_t}(p_t)=a]$.
\end{definition}

Similarly to Theorems \ref{thm:swap-regret-expectation} in the single-sequence case, decision calibration leads to the constrained swap regret bound.
\begin{theorem} \label{thm:swap-regret-expectation-subsequence}
    Let $\cS$ be a collection of subsequences. Let $\cN$ be a set of agents, where each agent is equipped with a utility function $u: \cA \times \cY \to [0,1]$ and $J$ constraint functions $\{c_j: \cA \times \cY \to [-1,1]\}_{j \in [J]}$. The utility functions are linear and $L$-Lipschitz in the second argument. Each agent runs Algorithm \ref{alg:elimination-expectation-subsequence} to compete with the benchmark class $\cA^{\E[\bc]}_S$ over every subsequence $S \in \cS$, where the threshold parameter for each subsequence $S$ is set to $4\sqrt{|S|\ln\frac{|\mathcal{A}||\mathcal{N}||\mathcal{S}|^2J|S|}{\delta}}$. If the sequence of predictions $p_1,\ldots,p_T$ is $(\cN,\cS,\beta_2)$-decision calibrated, then the $\cA^{\E[\bc]}_{S}$-constrained swap regret of any agent over any subsequence $S \in \cS$ is bounded by:
    \begin{align*}
        \Reg_\swap(u,\cA_{S}^{\E[\bc]},S) \le 2 L |\cA| \beta_2(|S|/|\cA|).
    \end{align*}
\end{theorem}

Similarly to the proof of Theorem \ref{thm:swap-regret-expectation}, for any subsequence $S \in \cS$ and action modification rule $\phi: \cA \to \cA_{S}^{\E[\bc]}$, we can decompose the regret against $\phi$ into three parts as:
\begin{align*}
    \MoveEqLeft \sum_{t \in S} \left( u(\phi(a_t),y_t) - u(a_t,y_t) \right) \\
    &= \sum_{t \in S} \left( u(\phi(a_t),y_t) - u(\phi(a_t),p_t) \right) + \sum_{t \in S} \left( u(\phi(a_t),p_t) - u(a_t,p_t) \right) + \sum_{t \in S} \left( u(a_t,p_t) - u(a_t,y_t) \right).
\end{align*}

Conditioning on the event in Corollary \ref{cor:benchmark-expectation-subsequence} and the fact that $\phi(a_t) \in \cA_S^{\E[\bc]}$,  $\phi(a_t) \in U_t$ for every $t \in S$. Since agents $U_t$-constrained best response to the prediction $p_t$ in round $t$, $a_t = \argmax_{b \in U_t} u(b,p_t)$. This implies that:
\begin{align*}
    u(\phi(a_t),p_t) \le u(a_t,p_t).
\end{align*}
So, the second part in the decomposition is upper bounded by 0.

Thus, it suffices to bound the other two parts using decision calibration.

\begin{lemma} \label{lem:decision-calibration-expectation-subsequence}
    If the sequence of predictions $p_1,\ldots,p_T$ is $(\cN,\cS,\beta_2)$-decision calibrated, then for any $(u,\bc) \in \cN$ and $S \in \cS$:
    \begin{align*}
        \left| \sum_{t \in S} (u(a_t,p_t) - u(a_t,y_t)) \right| \le L |\cA| \beta_2(|S|/|\cA|).
    \end{align*}
\end{lemma}
\begin{proof}
    Using the linearity of $u$, we can write:
    \begin{align*}
        \left| \sum_{t \in S} (u(a_t,p_t) - u(a_t,y_t)) \right| &= \left| \sum_{a \in \cA}\sum_{t=1}^T \1[t \in S,a_t=a] (u(a,p_t) - u(a,y_t)) \right| \\
        &= \left| \sum_{a \in \cA} \left( u\left(a, \sum_{t=1}^T \1[t \in S,a_t=a] p_t \right) - u\left(a, \sum_{t=1}^T \1[t \in S,a_t=a] y_t\right) \right)\right| \\
        &\le \sum_{a \in \cA} \left| u\left(a, \sum_{t=1}^T \1[t \in S,a_t=a] p_t \right) - u\left(a, \sum_{t=1}^T \1[t \in S,a_t=a] y_t\right) \right| \\
        &\le \sum_{a \in \cA} L \left\| \sum_{t=1}^T \1[t \in S,a_t=a] (p_t-y_t) \right\|_\infty \\
        &\le L \sum_{a \in \cA} \beta_2(T^{u,\E[\bc],S}(a)) \\
        &\le L |\cA| \beta_2(|S|/|\cA|).
    \end{align*}
    where the first inequality follows from the triangle inequality, the second inequality follows from $L$-Lipschitzness of $u$, the third inequality follows from $(\cN,\cS,\beta_2)$-decision calibration, and the fourth inequality follows from concavity of $\beta_2$.
\end{proof}

\begin{lemma} \label{lem:decision-calibration-expectation-subsequence-swap}
    If the sequence of predictions $p_1,\ldots,p_T$ is $(\cN,\cS,\beta_2)$-decision calibrated, then for any $(u,\bc) \in \cN$ and $S \in \cS$:
    \begin{align*}
        \left| \sum_{t \in S} (u(\phi(a_t),p_t) - u(\phi(a_t),y_t)) \right| \le L |\cA| \beta_2(|S|/|\cA|).
    \end{align*}
\end{lemma}
\begin{proof}
    Using the linearity of $u$, we can write:
    \begin{align*}
        \left| \sum_{t \in S} (u(\phi(a_t),p_t) - u(\phi(a_t),y_t)) \right| &= \left| \sum_{a \in \cA}\sum_{t=1}^T \1[t \in S,a_t=a] (u(\phi(a),p_t) - u(\phi(a),y_t)) \right| \\
        &= \left| \sum_{a \in \cA} \left( u\left(\phi(a), \sum_{t=1}^T \1[t \in S,a_t=a] p_t \right) - u\left(\phi(a), \sum_{t=1}^T \1[t \in S,a_t=a] y_t\right) \right)\right| \\
        &\le \sum_{a \in \cA} \left| u\left(\phi(a), \sum_{t=1}^T \1[t \in S,a_t=a] p_t \right) - u\left(\phi(a), \sum_{t=1}^T \1[t \in S,a_t=a] y_t\right) \right| \\
        &\le \sum_{a \in \cA} L \left\| \sum_{t=1}^T \1[t \in S,a_t=a] (p_t-y_t) \right\|_\infty \\
        &\le L \sum_{a \in \cA} \beta_2(T^{u,\E[\bc],S}(a)) \\
        &\le L |\cA| \beta_2(|S|/|\cA|).
    \end{align*}
    where the first inequality follows from the triangle inequality, the second inequality follows from $L$-Lipschitzness of $u$, the third inequality follows from $(\cN,\cS,\beta_2)$-decision calibration, and the fourth inequality follows from concavity of $\beta_2$.
\end{proof}

We can now complete the proof of Theorem \ref{thm:swap-regret-expectation-subsequence}.
\begin{proof}
    Fix any $(u,\bc) \in \cN$, $S \in \cS$, and $\phi: \cA \to \cA^{\E[\bc]}_{S}$. Applying Lemmas \ref{lem:decision-calibration-expectation-subsequence} and \ref{lem:decision-calibration-expectation-subsequence-swap} to the decomposition of the regret against $\phi$, we have that:
    \begin{align*}
        \sum_{t \in S} \left( u(a,y_t) - u(a_t,y_t) \right) &= \sum_{t \in S} \left( u(a,y_t) - u(a,p_t) \right) + \sum_{t \in S} \left( u(a,p_t) - u(a_t,p_t) \right) + \sum_{t \in S} \left( u(a_t,p_t) - u(a_t,y_t) \right) \\
        &\le 2 L |\cA| \beta_2(|S|/|\cA|).
\end{align*}
\end{proof}

We will again instantiate \textsc{Unbiased-Prediction} to make decision calibrated predictions.

\begin{theorem} \label{thm:unbiased-algorithm-expectation-subsequence}
    Let $\cS$ be a collection of subsequences. Let $\cN$ be a set of agents, where each agent is equipped with a utility function $u: \cA \times \cY \to [0,1]$ and $J$ constraint functions $\{c_j: \cA \times \cY \to [-1,1]\}_{j \in [J]}$. Each agent will run Algorithm \ref{alg:elimination-expectation-subsequence} to compete with the benchmark class $\cA^{\E[\bc]}_S$ over every subsequence $S \in \cS$, where the threshold parameter for each subsequence $S$ is set to $4\sqrt{|S|\ln\frac{|\mathcal{A}||\mathcal{N}||\mathcal{S}|^2J|S|}{\delta}}$. Each agent will maintain a candidate action set $\{U_t\}_{t \in [T]}$. 
    There is an instantiation of \textsc{Unbiased-Prediction} \citep{noarov2023highdimensional}
    ---which we call \textsc{Decision-Calibration-Subsequence}---
    producing predictions $p_1,...,p_T \in \cY$ satisfying, with probability at least $1-\delta$, for any $(u,\bc) \in \cN$, $a \in \cA$, and $S \in \cS$:
    \[
        \left\| \sum_{t=1}^T \1[t \in S, \CBR^u_{U_{t}}(p_t) = a] (p_t - y_t) \right\|_\infty \leq O\left( \ln(d|\cA||\cN||\cS|T) + \sqrt{\ln(d|\cA||\cN||\cS|T) \cdot T^{u,\E[\bc],S}(a)} + \sqrt{\ln\frac{d|\cA||\cN||\cS|}{\delta} \cdot |S|}\right).
    \]
    
\end{theorem}

Substituting the above bounds into Theorems \ref{thm:swap-regret-expectation-subsequence}, we arrive at the following corollary bounding the $\cA^{\E[\bc]}_{S}$-constrained swap regret over every subsequence $S \in \cS$.

\begin{corollary} \label{cor:swap-regret-expectation-subsequence}
    Let $\cN$ be a set of agents, where each agent is equipped with a utility function $u: \cA \times \cY \to [0,1]$ and $J$ constraint functions $\{c_j: \cA \times \cY \to [-1,1]\}_{j \in [J]}$. The utility functions are linear and $L$-Lipschitz in the second argument. Each agent will run Algorithm \ref{alg:elimination-expectation-subsequence} to compete with the benchmark class $\cA^{\E[\bc]}_S$ over every subsequence $S \in \cS$, where the threshold parameter for each subsequence $S$ is set to $4\sqrt{|S|\ln\frac{|\mathcal{A}||\mathcal{N}||\mathcal{S}|^2J|S|}{\delta}}$.
    The sequence of predictions $p_1,\ldots,p_T$ outputted by \textsc{Decision-Calibration-Subsequence} ensures that with probability at least $1-\delta$, the $\cA^{\E[\bc]}_{S}$-constrained swap regret of any agent over any subsequence $S \in \cS$ is bounded by:
    \begin{align*}
        \Reg_\swap(u,\cA_{S}^{\E[\bc]},S) \le O\left( L |\cA| \sqrt{|S| \ln\frac{d|\cA||\cN||\cS|T}{\delta}} \right).
    \end{align*}
\end{corollary}

Since constrained swap regret is stronger than constrained external regret, we can also bound the $\cA^{\E[\bc]}_{S}$-constrained external regret over every subsequence $S \in \cS$.
\begin{corollary} \label{cor:external-regret-expectation-subsequence}
    Let $\cN$ be a set of agents, where each agent is equipped with a utility function $u: \cA \times \cY \to [0,1]$ and $J$ constraint functions $\{c_j: \cA \times \cY \to [-1,1]\}_{j \in [J]}$. The utility functions are linear and $L$-Lipschitz in the second argument. Each agent will run Algorithm \ref{alg:elimination-expectation-subsequence} to compete with the benchmark class $\cA^{\E[\bc]}_S$ over every subsequence $S \in \cS$, where the threshold parameter for each subsequence $S$ is set to $4\sqrt{|S|\ln\frac{|\mathcal{A}||\mathcal{N}||\mathcal{S}|^2J|S|}{\delta}}$.
    The sequence of predictions $p_1,\ldots,p_T$ outputted by \textsc{Decision-Calibration-Subsequence} ensures that with probability at least $1-\delta$, the $\cA^{\E[\bc]}_{S}$-constrained external regret of any agent over any subsequence $S \in \cS$ is bounded by:
    \begin{align*}
        \Reg_\ext(u,\cA_{S}^{\E[\bc]},S) \le O\left( L |\cA| \sqrt{|S| \ln\frac{d|\cA||\cN||\cS|T}{\delta}} \right).
    \end{align*}
\end{corollary}

\section{Unbiased Prediction Algorithm} \label{app:unbiased-prediction}
In this section we present the $\textsc{Unbiased-Prediction}$ algorithm of \citet{noarov2023highdimensional}, from which our guarantees in Theorems \ref{thm:unbiased-algorithm-realization}, \ref{thm:unbiased-algorithm-realization-subsequence}, \ref{thm:unbiased-algorithm-expectation}, and \ref{thm:unbiased-algorithm-expectation-subsequence} follow.

We first introduce several notations and concepts from \citet{noarov2023highdimensional}. Let $\Pi = \{(x,p,y) \in \cX \times \cY \times \cY\}$ denote the set of possible realized triples at each round. An interaction over $T$ rounds produces a transcript $\pi_T \in \Pi^T$. We write $\pi^{<t}_T$ as the prefix of the first $t-1$ triples in $\pi_T$, for any $t \le T$. We write $\Pi^* = \cup_{T=1}^\infty \Pi^T$ for the space of all transcripts.
An \textit{event} $E\in\cE$ is a mapping from transcripts, contexts and predictions to $[0,1]$ , i.e. $E: \Pi^* \times \cX \times \cY \to [0,1]$. 

The $\textsc{Unbiased-Prediction}$ algorithm makes predictions that are unbiased conditional on a collection of events $\cE$. The algorithm's conditional bias guarantee depends logarithmically on the number of events:

\begin{theorem}\citep{noarov2023highdimensional} \label{thm:unbiased-prediction-algorithm}
For a collection of events $\cE$ and convex prediction/outcome space $\cY\subseteq [0,1]^d$, Algorithm \ref{alg:unbiased-prediction} produces predictions $\psi_1,...,\psi_T \in \Delta \cY$ such that for any sequence of outcomes $y_1,...,y_T \in \cY$ chosen by the adversary:
    \[
    \left\| \sum_{t=1}^T \E_{p_t\sim\psi_t}[E(\pi_T^{<t}, x_t, p_t)(p_t - y_t)] \right\|_\infty \leq O\left( \ln(d|\cE|T) + \sqrt{\ln(d|\cE|T) \cdot n_T(E) } \right).
    \]
    where $n_T(E) = \sum_{t=1}^T E(\pi_T^{<t},x_t,p_t)$. The algorithm can be implemented with per-round running time scaling polynomially in $d$ and $|\cE|$.
\end{theorem}

\begin{algorithm}[H]
    \For{$t=1$ \KwTo $T$}{
        Observe $x_t$\;
        
        Define the distribution $q_t \in \Delta [2d|\cE|]$ such that for $E \in \cE, i\in[d], \sigma\in \{\pm 1\}$,
        \[
        q_t^{E, i, \sigma} \propto \exp\left( \frac{\eta}{2} \sum_{s=1}^{t-1} \sigma \cdot \E_{p_s\sim\psi_s}[E(\pi_T^{<s}, x_s, p_s) (p_s^i - y_s^i)] \right);
        \]

        Output the solution to the minmax problem:
        \[
        \psi_t \gets \argmin_{\psi_t' \in \Delta\cY} \max_{y \in \cY} \E_{p_t \sim \psi_t'}\left[\sum_{E, i, \sigma} q_t(E, i, \sigma) \cdot \sigma \cdot E(\pi_T^{<s}, x_s, p_s) \cdot (p_s^i - y_s^i) \right];
        \]        
    }    
    \caption{$\textsc{Unbiased-Prediction}$}
    \label{alg:unbiased-prediction}
\end{algorithm}

Theorem \ref{thm:unbiased-prediction-algorithm} is stated as expected error bounds over randomized predictions $\psi_t\in\Delta\cY$. In the following Corollary \ref{cor:unbiased-prediction-algorithm}, we state the guarantees based on realized predictions $p_t$ that are sampled from $\psi_t$. Our guarantees in Theorems \ref{thm:unbiased-algorithm-realization} and \ref{thm:unbiased-algorithm-realization-subsequence} directly follow from Corollary \ref{cor:unbiased-prediction-algorithm}.
\begin{corollary} \label{cor:unbiased-prediction-algorithm}
For a collection of events $\cE$ and convex prediction/outcome space $\cY\subseteq [0,1]^d$, Algorithm \ref{alg:unbiased-prediction} produces predictions $p_1,...,p_T \in \cY$ such that for any sequence of outcomes $y_1,...,y_T \in \cY$ chosen by the adversary, with probability at least $1-\delta$:
    \[
        \left\| \sum_{t=1}^T E(\pi_T^{<t}, x_t, p_t)(p_t - y_t) \right\|_\infty \leq O\left( \ln(d|\cE|T) + \sqrt{\ln(d|\cE|T) \cdot n_T(E) } + \sqrt{\ln(d|\cE| / \delta) \cdot T} \right).
    \]
    where $n_T(E) = \sum_{t=1}^T E(\pi_T^{<t},x_t,p_t)$. The algorithm can be implemented with per-round running time scaling polynomially in $d$ and $|\cE|$.
\end{corollary}
\begin{proof}
    Fix any $m \in [d]$ and $E \in \cE$. Consider the sequence $\{E(\pi_T^{<t}, x_t, p_t)(p_{t,m} - y_{t,m}) - \E_{p_t\sim\psi_t}[E(\pi_T^{<t}, x_t, p_t)(p_{t,m} - y_{t,m})]\}_{t=1}^T$, where $p_{t,m}$ and $y_{t,m}$ are the $m$-th coordinate of $p_t$ and $y_t$, respectively. It is a sequence of martingale differences, since for any $t \in [T]$: 
    \begin{align*}
        \E\left[ E(\pi_T^{<t}, x_t, p_t)(p_{t,m} - y_{t,m}) - \E_{p_t\sim\psi_t}[E(\pi_T^{<t}, x_t, p_t)(p_{t,m} - y_{t,m})] \mid \sigma(\pi_T^{<t})\right] = 0.
    \end{align*}

    By Azuma-Hoeffding inequality (Lemma \ref{lem:azuma}), we have that with probability at least $1-\frac{\delta}{d|\cE|}$:
    \begin{align*}
       \left| \sum_{t=1}^T E(\pi_T^{<t}, x_t, p_t)(p_{t,m} - y_{t,m}) - \sum_{t=1}^T \E_{p_t\sim\psi_t}[E(\pi_T^{<t}, x_t, p_t)(p_{t,m} - y_{t,m})] \right| \le 2\sqrt{2 \ln (2d|\cE| / \delta) \cdot T}.
    \end{align*}

    Using the union bound over all $m \in [d]$ and $E \in \cE$, we have that with probability at least $1-\delta$, for any $m \in [d]$ and any $E \in \cE$:
    \begin{align*}
       \left\| \sum_{t=1}^T E(\pi_T^{<t}, x_t, p_t)(p_{t} - y_{t}) - \sum_{t=1}^T \E_{p_t\sim\psi_t}[E(\pi_T^{<t}, x_t, p_t)(p_{t} - y_{t})] \right\|_\infty \le 2\sqrt{2 \ln (2d|\cE| / \delta) \cdot T}.
    \end{align*}

    As a result,
    \begin{align*}
        \left\| \sum_{t=1}^T E(\pi_T^{<t}, x_t, p_t)(p_t - y_t) \right\|_\infty &\le \left\| \sum_{t=1}^T \E_{p_t\sim\psi_t}[E(\pi_T^{<t}, x_t, p_t)(p_t - y_t)] \right\|_\infty \\
        &\quad + \left\| \sum_{t=1}^T E(\pi_T^{<t}, x_t, p_t)(p_{t} - y_{t}) - \sum_{t=1}^T \E_{p_t\sim\psi_t}[E(\pi_T^{<t}, x_t, p_t)(p_{t} - y_{t})] \right\|_\infty \\
        &\le O\left( \ln(d|\cE|T) + \sqrt{\ln(d|\cE|T) \cdot n_T(E) } \right) + 2\sqrt{2 \ln (2d|\cE| / \delta) \cdot T} \\
        &= O\left( \ln(d|\cE|T) + \sqrt{\ln(d|\cE|T) \cdot n_T(E) } + \sqrt{\ln (d|\cE| / \delta) \cdot T}\right).
    \end{align*}
\end{proof}

We further extend the guarantee to the multi-subsequence setting in Corollary \ref{cor:unbiased-prediction-algorithm-subsequence}. Our guarantees in Theorems \ref{thm:unbiased-algorithm-expectation} and \ref{thm:unbiased-algorithm-expectation-subsequence} directly follow from Corollary \ref{cor:unbiased-prediction-algorithm-subsequence}.
\begin{corollary} \label{cor:unbiased-prediction-algorithm-subsequence}
Let $\cS$ be a collection of subsequences. For a collection of events $\cE$ and convex prediction/outcome space $\cY\subseteq [0,1]^d$, Algorithm \ref{alg:unbiased-prediction} instantiated with the collection of events $\{\1[t \in S] \cdot E\}_{S \in \cS, E \in \cE}$ produces predictions $p_1,...,p_T \in \cY$ such that for any sequence of outcomes $y_1,...,y_T \in \cY$ chosen by the adversary, with probability at least $1-\delta$, for any $S \in \cS$:
    \[
        \left\| \sum_{t \in S} E(\pi_T^{<t}, x_t, p_t)(p_t - y_t) \right\|_\infty \leq O\left( \ln(d|\cE||\cS|T) + \sqrt{\ln(d|\cE||\cS|T) \cdot n_S(E) } + \sqrt{\ln(d|\cE||\cS| / \delta) \cdot |S|} \right).
    \]
    where $n_S(E) = \sum_{t \in S} E(\pi_T^{<t},x_t,p_t)$. The algorithm can be implemented with per-round running time scaling polynomially in $d$, $|\cE|$, and $|\cS|$.
\end{corollary}
\begin{proof}
    Fix any $m \in [d]$, $E \in \cE$, and $S \in \cS$. Consider the sequence $\{E(\pi_T^{<t}, x_t, p_t)(p_{t,m} - y_{t,m}) - \E_{p_t\sim\psi_t}[E(\pi_T^{<t}, x_t, p_t)(p_{t,m} - y_{t,m})]\}_{t=1}^T$, where $p_{t,m}$ and $y_{t,m}$ are the $m$-th coordinate of $p_t$ and $y_t$, respectively. It is a sequence of martingale differences, since for any $t \in [T]$: 
    \begin{align*}
        \E\left[ E(\pi_T^{<t}, x_t, p_t)(p_{t,m} - y_{t,m}) - \E_{p_t\sim\psi_t}[E(\pi_T^{<t}, x_t, p_t)(p_{t,m} - y_{t,m})] \mid \sigma(\pi_T^{<t},x_t)\right] = 0.
    \end{align*}

    The subsequence of these terms corresponding to rounds $s \in S$, i.e., $\{E(\pi_T^{<s}, x_s, p_s)(p_{s,m} - y_{s,m}) - \E_{p_s\sim\psi_s}[E(\pi_T^{<s}, x_s, p_s)(p_{s,m} - y_{s,m})]\}_{s \in S: s \le t}$, is also a martingale difference sequence, because the selection rule is predictable with respect to the filtration $\sigma(\pi_T^{<s},x_s)$.

    By Azuma-Hoeffding inequality (Lemma \ref{lem:azuma}), we have that with probability at least $1-\frac{\delta}{d|\cE||\cS|}$:
    \begin{align*}
       \left| \sum_{t \in S} E(\pi_T^{<t}, x_t, p_t)(p_{t,m} - y_{t,m}) - \sum_{t \in S} \E_{p_t\sim\psi_t}[E(\pi_T^{<t}, x_t, p_t)(p_{t,m} - y_{t,m})] \right| \le 2\sqrt{2 \ln (2d|\cE||\cS| / \delta) \cdot |S|}.
    \end{align*}

    Using the union bound over all $m \in [d]$, $E \in \cE$, and $S \in \cS$, we have that with probability at least $1-\delta$, for any $m \in [d]$, any $E \in \cE$, and any $S \in \cS$:
    \begin{align*}
       \left\| \sum_{t \in S} E(\pi_T^{<t}, x_t, p_t)(p_{t} - y_{t}) - \sum_{t \in S} \E_{p_t\sim\psi_t}[E(\pi_T^{<t}, x_t, p_t)(p_{t} - y_{t})] \right\|_\infty \le 2\sqrt{2 \ln (2d|\cE| / \delta) \cdot |S|}.
    \end{align*}

    By Theorem \ref{thm:unbiased-prediction-algorithm}, for any $E \in \cE$ and any $S \in \cS$:
    \begin{align*}
        \left\| \sum_{t \in S} \E_{p_t\sim\psi_t}[E(\pi_T^{<t}, x_t, p_t)(p_t - y_t)] \right\|_\infty &= \left\| \sum_{t=1}^T \E_{p_t\sim\psi_t}[\1[t \in S]E(\pi_T^{<t}, x_t, p_t)(p_t - y_t)] \right\|_\infty \\
        &\leq O\left( \ln(d|\cE||\cS|T) + \sqrt{\ln(d|\cE||\cS|T) \cdot n_S(E) } \right).
    \end{align*}
    
    Therefore,
    \begin{align*}
        \left\| \sum_{t \in S} E(\pi_T^{<t}, x_t, p_t)(p_t - y_t) \right\|_\infty &\le \left\| \sum_{t \in S} \E_{p_t\sim\psi_t}[E(\pi_T^{<t}, x_t, p_t)(p_t - y_t)] \right\|_\infty \\
        &\quad + \left\| \sum_{t \in S} E(\pi_T^{<t}, x_t, p_t)(p_{t} - y_{t}) - \sum_{t \in S} \E_{p_t\sim\psi_t}[E(\pi_T^{<t}, x_t, p_t)(p_{t} - y_{t})] \right\|_\infty \\
        &\le O\left( \ln(d|\cE||\cS|T) + \sqrt{\ln(d|\cE||\cS|T) \cdot n_S(E) } \right) + 2\sqrt{2 \ln (2d|\cE||\cS| / \delta) \cdot |S|} \\
        &= O\left( \ln(d|\cE||\cS|T) + \sqrt{\ln(d|\cE||\cS|T) \cdot n_S(E) } + \sqrt{\ln (d|\cE||\cS| / \delta) \cdot |S|}\right).
    \end{align*}
\end{proof}

Theorem \ref{thm:unbiased-algorithm-realization} instantiates Algorithm \ref{alg:unbiased-prediction} with the collection of events $\cE = \{ \1[\CBR^u_{\widehat\cA^\bc_{t}}(p_t) = a] \}_{(u,\bc) \in \cN, a\in\cA}$.
Thus, $|\cE| = |\cA||\cN|$, from which our guarantees follow. 

Similarly, Theorem \ref{thm:unbiased-algorithm-expectation} instantiates Algorithm \ref{alg:unbiased-prediction} with the collection of events $\cE = \{ \1[\CBR^u_{\widehat\cA^{\E[\bc]}_{t}}(p_t) = a] \}_{(u,\bc) \in \cN, a\in\cA}$.
Thus, $|\cE| = |\cA||\cN|$ as well.

% In the multi-subsequence setting, Theorems \ref{thm:unbiased-algorithm-realization-subsequence} and \ref{thm:unbiased-algorithm-expectation-subsequence} instantiate Algorithm \ref{alg:unbiased-prediction} with the collection of events $\cE = \{ \1[t \in S, \CBR^u_{U_t}(p_t) = a] \}_{(u,\bc) \in \cN, a\in\cA, S \in \cS}$. Thus, $|\cE| = |\cA||\cN||\cS|$.
In the multi-subsequence setting, Theorems \ref{thm:unbiased-algorithm-realization-subsequence} and \ref{thm:unbiased-algorithm-expectation-subsequence} instantiate Algorithm \ref{alg:unbiased-prediction} with the collection of events $\cE = \{ \1[\CBR^u_{U_t}(p_t) = a] \}_{(u,\bc) \in \cN, a\in\cA}$. Thus, $|\cE| = |\cA||\cN|$.

% We note that $\widehat\cA_t^\bc$, $\widehat\cA_t^{\E[\bc]}$, and $U_t$ are constructed based on the actions' performances in the first $t-1$ rounds. Therefore, $\1[\CBR^u_{\widehat\cA^\bc_{t}}(p_t) = a]$, $\1[\CBR^u_{\widehat\cA^{\E[\bc]}_{t}}(p_t) = a]$, and $\1[t \in S, \CBR^u_{U_t}(p_t) = a]$ are indeed functions of $\pi_T^{<t}$ and $p_t$, hence they qualify for the events in Theorem \ref{thm:unbiased-prediction-algorithm}.
We note that $\widehat\cA_t^\bc$, $\widehat\cA_t^{\E[\bc]}$, and $U_t$ are constructed based on the actions' performances in the first $t-1$ rounds. Therefore, $\1[\CBR^u_{\widehat\cA^\bc_{t}}(p_t) = a]$, $\1[\CBR^u_{\widehat\cA^{\E[\bc]}_{t}}(p_t) = a]$, and $\1[\CBR^u_{U_t}(p_t) = a]$ are indeed functions of $\pi_T^{<t}$ and $p_t$, hence they qualify for the events in Theorem \ref{thm:unbiased-prediction-algorithm}.

\section{Azuma-Hoeffding's Inequality}
% \begin{lemma} \label{lem:azuma}
% If $X_1, \ldots, X_T$ is a martingale difference sequence, and for every $t$, with probability $1,\left|X_t\right| \leq M$. Then with probability at least $1-\delta$:
% $$
% \left|\sum_{t=1}^T X_t\right| \leq M \sqrt{2 T \ln \frac{2}{\delta}}
% $$
% \end{lemma}

% \begin{lemma} \label{lem:azuma}
% If $X_1, \ldots, X_T$ is a martingale difference sequence, and for every $t$, with probability $1,\left|X_t\right| \leq M$. Then with probability at least $1-2e^{-\frac{\epsilon^2}{2M^2T}}$:
% $$
% \left|\sum_{t=1}^T X_t\right| \leq \epsilon
% $$
% \end{lemma}

\begin{lemma} \label{lem:azuma}
If $X_1, \ldots, X_T$ is a martingale difference sequence, and for every $t$, with probability $1,\left|X_t\right| \leq M_t$. Then with probability at least $1-2e^{-\frac{\epsilon^2}{2\sum_{t=1}^T M_t^2}}$:
$$
\left|\sum_{t=1}^T X_t\right| \leq \epsilon.
$$
\end{lemma}

\end{document}